\documentclass{article}

\usepackage[margin=1in]{geometry}
\usepackage{lmodern}
\usepackage{natbib}
\usepackage[utf8]{inputenc}
\usepackage[T1]{fontenc}
\usepackage{hyperref}
\hypersetup{hidelinks}
\usepackage{url}
\usepackage{booktabs}
\usepackage{amsfonts}
\usepackage{nicefrac}
\usepackage{microtype}
\usepackage{xcolor}
\usepackage{graphicx}
\graphicspath{{figures/}{./}}
\usepackage{amsmath}
\usepackage{amssymb}
\usepackage{amsthm}
\newtheorem{proposition}{Proposition}
\usepackage{multicol}
\usepackage{tabularx}
\usepackage{array}
\usepackage{subcaption}
\usepackage{algorithm}
\usepackage{algorithmic}
\title{Time-partitioning Gaussian Process Mixture Model}
\author{
  Guilherme Piantino\thanks{\texttt{guilhermejlp@al.insper.edu.br}} \\
  Insper Institute of Education \\
  and Research, S\~ao Paulo
  \and
  Hedibert F.~Lopes \\
  Insper Institute of Education \\
  and Research, S\~ao Paulo
}
\date{}

\newcommand{\SimKHom}{10.3}
\newcommand{\SimKsdHom}{0.7}
\newcommand{\SimKPow}{10.0}
\newcommand{\SimKsdPow}{0.3}
\newcommand{\SimFoneHom}{0.981}
\newcommand{\SimFonePow}{0.996}
\newcommand{\SimExactHom}{72}
\newcommand{\SimExactPow}{92}

\newcommand{\SimMaePow}{0.05}
\newcommand{\SigHom}{0.121}
\newcommand{\SigPow}{0.526}

\newcommand{\SimSnAlt}{0.095}
\newcommand{\SimFoneHomAlt}{0.996}

\newcommand{\SimExactHomAlt}{96}

\newcommand{\SigHomMinAll}{0.104}
\newcommand{\SigHomMaxAll}{0.148}
\newcommand{\SigPowMinAll}{0.516}
\newcommand{\SigPowMaxAll}{0.540}
\newcommand{\SimFoneHomRbfSn}{0.984}
\newcommand{\SimFonePowRbfSn}{0.997}

\newcommand{\SimKHomRbfSn}{10.3}
\newcommand{\SimSnRbfSn}{0.100}

\newcommand{\VixK}{13.4}
\newcommand{\VixKlo}{11}
\newcommand{\VixKhi}{17}
\newcommand{\VixModalK}{11}
\newcommand{\VixModalShare}{0.19}
\newcommand{\VixSnAlt}{0.049}

\newcommand{\VixSnRbf}{0.063}

\newcommand{\VixKAlt}{8.7}
\newcommand{\TrKAlt}{8.2}
\newcommand{\VixKRbfSn}{29.0}
\newcommand{\TrKRbfSn}{40.2}

\newcommand{\TrK}{7.0}
\newcommand{\TrKlo}{5}
\newcommand{\TrKhi}{10}
\newcommand{\TrModalK}{5}
\newcommand{\FlexKvix}{51}
\newcommand{\FlexKtr}{13}
\newcommand{\RhatMainVix}{1.031}
\newcommand{\RhatMainTr}{1.011}

\newcommand{\RhatRbfSnTr}{1.14}
\newcommand{\EssKRbfSn}{12}
\newcommand{\RhatRbfSn}{1.27}
\newcommand{\EssKvix}{143}
\newcommand{\EssKtr}{387}
\newcommand{\MinPerChainVix}{8}
\newcommand{\MinPerChainVixAlt}{50}
\newcommand{\MinPerChainVixRbfSn}{5}
\newcommand{\MinPerChainVixFlex}{7}

\newcommand{\PriorTV}{0.014}
\newcommand{\VixCpA}{2015-06-29}
\newcommand{\VixCpPA}{0.96}
\newcommand{\VixCpB}{2015-08-21}
\newcommand{\VixCpPB}{1.00}
\newcommand{\VixCpC}{2016-11-09}
\newcommand{\VixCpPC}{1.00}
\newcommand{\VixCpD}{2017-04-24}
\newcommand{\VixCpPD}{0.99}
\newcommand{\VixCpE}{2017-08-10}
\newcommand{\VixCpPE}{0.98}
\newcommand{\VixCpF}{2018-02-05}
\newcommand{\VixCpPF}{1.00}
\newcommand{\VixCpG}{2020-03-02}
\newcommand{\VixCpPG}{0.99}
\newcommand{\VixCpH}{2021-01-27}
\newcommand{\VixCpPH}{0.99}
\newcommand{\VixCpI}{2021-12-07}
\newcommand{\VixCpPI}{0.68}
\newcommand{\VixCpJ}{2022-04-22}
\newcommand{\VixCpPJ}{0.42}
\newcommand{\InsMainVixRmse}{0.208}
\newcommand{\InsMainVixCrps}{0.122}
\newcommand{\FcMainVixHone}{0.706}
\newcommand{\FcMainVixHfive}{1.400}
\newcommand{\InsMainTrRmse}{0.061}
\newcommand{\InsMainTrCrps}{0.037}

\newcommand{\FcMainTrHfive}{0.729}
\newcommand{\InsFlexVixRmse}{0.076}
\newcommand{\InsFlexVixCrps}{0.046}
\newcommand{\FcFlexVixHone}{0.476}
\newcommand{\FcFlexVixHfive}{0.940}
\newcommand{\InsFlexTrRmse}{0.021}
\newcommand{\InsFlexTrCrps}{0.013}
\newcommand{\FcFlexTrHone}{0.250}
\newcommand{\FcFlexTrHfive}{0.595}

\newcommand{\InsAltVixCrps}{0.068}

\newcommand{\InsRbfSnVixCrps}{0.093}

\newcommand{\FcStdVixHone}{0.470}

\newcommand{\FcStdRVixHone}{0.503}

\newcommand{\FcStdVixHfive}{0.946}
\newcommand{\FcStdCrpsVixHfive}{0.496}
\newcommand{\FcStdRVixHfive}{1.003}
\newcommand{\FcFlexCrpsVixHfive}{0.482}

\newcommand{\InsGprVixRmse}{0.146}
\newcommand{\InsGprVixCrps}{0.081}
\newcommand{\InsGpmVixRmse}{0.115}
\newcommand{\InsGpmVixCrps}{0.065}

\newcommand{\RiceVix}{0.056}
\newcommand{\GsjVix}{0.048}

\newcommand{\FcStdTrHone}{0.248}

\newcommand{\FcStdTrHfive}{0.614}
\newcommand{\FcStdCrpsTrHfive}{0.347}
\newcommand{\FcStdRTrHfive}{0.909}
\newcommand{\FcFlexCrpsTrHfive}{0.377}

\newcommand{\InsGprTrCrps}{0.025}

\newcommand{\InsGpmTrCrps}{0.018}

\newcommand{\RiceTr}{0.041}
\newcommand{\GsjTr}{0.034}

\newcommand{\EvVixA}{1.00}
\newcommand{\EvVixB}{0.00}
\newcommand{\EvVixC}{1.00}
\newcommand{\EvVixD}{0.89}
\newcommand{\EvVixE}{0.01}
\newcommand{\EvVixF}{0.00}
\newcommand{\EvTrA}{0.22}
\newcommand{\EvTrB}{1.00}
\newcommand{\EvTrC}{0.11}
\newcommand{\EvAltVixA}{1.00}

\newcommand{\EvAltVixC}{1.00}
\newcommand{\EvAltVixD}{0.00}

\newcommand{\EvAltTrA}{0.01}

\newcommand{\ExFineGVix}{1{,}760}
\newcommand{\ExFineKVix}{13.69}
\newcommand{\ExMcmcKVix}{13.44}
\newcommand{\ExMcseVix}{0.15}
\newcommand{\ExTvVix}{0.049}
\newcommand{\ExDmaxVix}{0.14}
\newcommand{\ExDmeanVix}{0.007}
\newcommand{\ExSigVix}{0.082}
\newcommand{\ExSigMcmcVix}{0.090}
\newcommand{\ExThVix}{3.58}
\newcommand{\ExThMcmcVix}{3.56}
\newcommand{\ExCoarseKVix}{14.43}
\newcommand{\ExCoarseGVix}{210}

\newcommand{\ExFineKTr}{6.95}
\newcommand{\ExMcmcKTr}{7.00}
\newcommand{\ExMcseTr}{0.07}
\newcommand{\ExTvTr}{0.024}
\newcommand{\ExDmaxTr}{0.04}
\newcommand{\ExDmeanTr}{0.002}
\newcommand{\ExSigTr}{0.367}
\newcommand{\ExSigMcmcTr}{0.375}
\newcommand{\ExThTr}{5.16}
\newcommand{\ExThMcmcTr}{5.26}

\newcommand{\PeaksTpHalf}{7}
\newcommand{\PeaksTpFifth}{13}

\newcommand{\PeaksBocpdFifth}{20}

\newcommand{\PeaksHdpFifth}{38}
\newcommand{\PeaksMmAHalf}{7}

\newcommand{\PeaksMmBHalf}{3}
\newcommand{\PeaksMmBFifth}{14}
\newcommand{\SimFoneHomMmA}{0.985}
\newcommand{\SimKHomMmA}{9.8}
\newcommand{\SimKsdHomMmA}{0.5}
\newcommand{\SimMaeHomMmA}{3.38}

\newcommand{\SimFonePowMmA}{0.987}
\newcommand{\SimKPowMmA}{9.8}
\newcommand{\SimKsdPowMmA}{0.4}
\newcommand{\SimMaePowMmA}{1.22}

\newcommand{\MmKVixA}{9.8}
\newcommand{\InsMmAVixRmse}{0.224}
\newcommand{\InsMmAVixCrps}{0.121}
\newcommand{\MmKVixB}{16.7}

\newcommand{\MmPhiVix}{0.997}

\newcommand{\InsMmATrRmse}{0.059}
\newcommand{\InsMmATrCrps}{0.032}

\newcommand{\MmPhiTr}{0.998}
\newcommand{\EvMmAVixA}{0.00}

\newcommand{\EvMmAVixD}{0.00}

\begin{document}

\maketitle

\begin{abstract}
Nonstationarity and heterogeneous dynamics are common in financial and macroeconomic time series. Volatility alternates between calm and turbulent regimes, interest rates follow monetary-policy cycles, and major events trigger persistent structural shifts. We propose the Time-partitioning Gaussian Process Mixture Model (TPGPMM), a Bayesian nonparametric framework for regime identification and nonstationary prediction that infers such regimes and their within-regime dependence from data, without Markovian assumptions. It combines (i) an adaptive Pitman--Yor process prior on ordered compositions, mapped to block lengths with contiguity enforced by construction, and (ii) regime-specific Gaussian processes for temporal dependence, yielding a nonstationary predictive distribution. By inferring its discount parameter, the prior adapts to homogeneous and highly heterogeneous regime lengths.

Unlike sequential or Markov-switching methods, it flags persistent structural shifts as change points rather than reacting to transient shocks. A split--merge Metropolis-within-Gibbs sampler with exact Gibbs updates of regime boundaries jointly learns change points and hyperparameters in batch, quantifying uncertainty over the partition and within-regime dynamics. With an RBF kernel and a fixed, data-informed noise level, TPGPMM produces a sparse, confident segmentation. On simulated data it recovers the true regimes, and on the CBOE Volatility Index and the 10-year US Treasury yield it identifies regimes aligned with documented market and monetary-policy events. For prediction, variants estimating the noise level fit better than change-point baselines and fit and forecast better than the main specification, at the cost of a less sparse or less event-aligned segmentation; with a regime mean, they fit better than a fitted stationary Gaussian process and have lower forecast RMSE except at the shortest horizons.
\end{abstract}

\section{Introduction}

Macro-financial time series rarely obey a single stable data-generating process. Volatility moves between calm and turbulent regimes, interest rates track monetary-policy cycles, and both are punctuated by structural breaks around major economic events. Identifying these regimes---and separating persistent shifts from transient shocks---is central to risk management, forecasting, and policy analysis, yet the number, timing, and duration of regimes are typically unknown a priori and must be inferred from the data. Mixture models offer a natural starting point, generating each observation from one of several latent components \citep{McLachlanPeel2000}; but classical finite mixtures require fixing the number of components in advance. Bayesian nonparametric methods relax this by placing a prior on an infinite mixture, so that the effective number of components is inferred from the data. A commonly used prior is the Dirichlet process (DP), introduced by \citet{Ferguson1973}, which induces the Dirichlet process mixture (DPM), with posterior computation and density estimation developed by \citet{EscobarWest1995} and \citet{Neal2000}. In the Gaussian case this yields the flexible, complexity-adaptive Dirichlet process Gaussian mixture model (DPGMM) \citep{GorurRasmussen2010,Rasmussen1999}.

In time-series applications, partition-based Bayesian nonparametric methods have often been applied in the context of outlier detection, as in \citet{QuintanaIglesias2003}. This is largely because standard DP-mixture formulations rely on exchangeability, an assumption that is inappropriate for ordered data where temporal dependence is intrinsic. To overcome this limitation, researchers have extended these frameworks to allow the random mixing distribution to evolve with covariates. Key developments include the dependent Dirichlet process \citep{MacEachern2026}, order-based dependent Dirichlet processes \citep{GriffinSteel2006}, and kernel stick-breaking processes \citep{DunsonPark2008}.

In contrast, much of the time-series regime literature is built on state-space and switching formulations---including Markov-switching autoregressions \citep{Hamilton1989} and hidden Markov models \citep{Rabiner1989}. Bayesian nonparametric extensions allow an unknown number of regimes while typically retaining a Markovian latent state structure \citep{FoxSudderthJordanWillsky2011,KoChongGhosh2015}. Partition-based change-point models \citep{BarryHartigan1993,Chib1998,Fearnhead2006} and related exchangeable partition probability function (EPPF)--based constructions also often adopt simple or Markovian within-segment models; a notable example is \citet{MartinezMena2014}, who combine a prior on ordered compositions with discretely observed Ornstein--Uhlenbeck diffusion processes.

Motivated by these limitations, we develop the Time-partitioning Gaussian Process Mixture Model (TPGPMM), a Bayesian nonparametric framework designed to capture multiple regime changes and nonstationary dependence in time-ordered data without imposing a rigid Markovian structure. Standard Gaussian processes (GPs) provide a robust probabilistic framework for nonparametric regression, but the stationarity of their covariance kernels limits their ability to represent regime shifts and abrupt changes. The TPGPMM addresses this limitation by inferring both the number and locations of regimes directly from the data, representing nonstationarity through a flexible piecewise GP structure.

TPGPMM is designed for two purposes. The first is \emph{regime identification}: the posterior over contiguous partitions gives the number and timing of regime changes, with their uncertainty, and separates persistent shifts from transient shocks, which is what risk management and policy analysis require. The second is \emph{nonstationary prediction}: given the partition, the regime-specific GPs form a nonstationary GP, and averaging over the posterior gives a mixture of such GPs whose predictive distribution adapts to the regime structure, both within the sample and beyond its end. The two purposes favour different choices of the observation-noise level. A conservative fixed level yields a sparse segmentation, whereas estimating the noise yields better fit and forecasts, so we evaluate the model for each purpose separately.

This connects our work to the broader literature on nonstationary GPs. Nonstationarity is often introduced via spatially varying kernels, such as process convolution constructions \citep{HigdonEtAl1999} and nonstationary Mat\'ern-type formulations \citep{PaciorekSchervish2003,Katzfuss2013}, or by modelling input-dependent hyperparameters, including heteroskedastic GP frameworks \citep{HeinonenEtAl2016,BinoisEtAl2018}. Another line of work partitions the input space to fit local GPs with distinct hyperparameters, using approaches such as Voronoi tessellations \citep{KimEtAl2005}, treed GPs \citep{GramacyLee2008}, and local approximate GPs \citep{GramacyApley2015}, which constructs a separate local GP approximation around each prediction location. A complementary strategy learns nonlinear transformations (warpings) of the input space to induce stationarity in a latent representation, originating with \citet{SampsonGuttorp1992} and \citet{SchmidtOHagan2003}. Building on this, \citet{DamianouLawrence2013} introduced deep Gaussian processes (DGPs), which compose multiple GP layers to learn hierarchical warpings; recent work continues to develop full Bayesian inference for these models \citep{Sauer2023}.

Our work is most closely related to mixtures and segmentations of GPs. \citet{RasmussenGhahramani2001} proposed infinite mixtures of GP experts using DP-based mixing, while \citet{Tresp2000} studied mixtures of GPs with gating functions. However, mixture-of-experts formulations generally allow non-contiguous assignments and are not tailored to contiguous regime segmentation in ordered data. In the change-point setting, \citet{SaatciTurnerRasmussen2010} and \citet{GarnettEtAl2010} introduced sequential online Bayesian methods in which each segment is modelled by a GP and change-points are governed by a hazard function. Compared to these alternatives, TPGPMM differs in two key respects: (i) the partition prior is exchangeable on compositions and position-neutral, so segmentation is driven entirely by the block-size distribution rather than by a Markovian run-length recursion; and (ii) inference is performed in batch, jointly across the entire series, so each candidate boundary is evaluated using global posterior evidence. As we show in Section~\ref{sec:regimes}, this combination identifies persistent structural shifts as change points rather than reacting to every transient volatility shock. We evaluate TPGPMM on simulated benchmarks and on two macro-financial series, the Chicago Board Options Exchange (CBOE) Volatility Index (VIX) and the 10-year United States (US) Treasury yield. For regime identification (Section~\ref{sec:regimes}), the main specification, with an RBF kernel and a fixed noise level, recovers a sparse segmentation aligned with major market and monetary-policy events. For nonstationary prediction (Section~\ref{sec:prediction}), the specifications that estimate the noise level fit better than the change-point baselines and fit and forecast better than the main one, at the cost of a less sparse or less event-aligned segmentation, and with a regime mean they fit better than a fitted stationary GP and have lower forecast RMSE except at the shortest horizons.

\section{The model}
\label{sec:model}

We observe a time-ordered sequence $\mathbf{y}=(Y_1,\ldots,Y_n)$ at fixed, equally spaced inputs $X_t=t-1$. Dependence among the observations is introduced through an ordered partition of the index set into $K$ contiguous blocks, each of which is interpreted as a distinct regime. We represent the segmentation by its composition:
\begin{equation}\label{eq:rho}
  \rho=(n_1,\ldots,n_K),
  \qquad n_k\in\{m_{\min},m_{\min}+1,\ldots\},
  \qquad \sum_{k=1}^K n_k=n,
\end{equation}
where $m_{\min}$ is a minimum regime length ($m_{\min}=3$ throughout). Given $\rho$, we define the cumulative endpoints
\[
  i_0=0,\qquad i_k=\sum_{j=1}^k n_j,\quad k=1,\ldots,K,
\]
so that $0=i_0<i_1<\cdots<i_K=n$. The internal regime boundaries are $\{i_1,\ldots,i_{K-1}\}$, and the corresponding contiguous blocks are
\begin{equation}\label{eq:Sk}
  S_k=\{i_{k-1}+1,\ldots,i_k\},\qquad n_k=|S_k|.
\end{equation}

Each block $S_k$ is associated with block-specific hyperparameters $\boldsymbol{\phi}_k\in\Phi$. We define the time-indexed hyperparameter process $\{\boldsymbol{\phi}_t\}_{t=1}^n$ by
\[
  \boldsymbol{\phi}_t = \boldsymbol{\phi}_k, \qquad t\in S_k,
\]
so that $\boldsymbol{\phi}_t$ is piecewise constant over time and can change only at the internal boundaries. Conditional on $\rho$, we collect the block-specific hyperparameters as $\boldsymbol{\phi}_{1:K}=(\boldsymbol{\phi}_1,\ldots,\boldsymbol{\phi}_K)$ and assume
\begin{equation}\label{eq:phi_factorisation}
  p(\boldsymbol{\phi}_{1:K}\mid\rho)=\prod_{k=1}^K p(\boldsymbol{\phi}_k),
\end{equation}
so that the block hyperparameters are independent and identically distributed given the partition.

For each block $k$, let $f_k(\cdot)$ denote the latent function with GP prior governed by $\boldsymbol{\phi}_k$, and define the latent evaluations
\(
  f_t = f_k(X_t),\ t\in S_k,
\)
with global latent vector $\mathbf{f}=(f_1,\ldots,f_n)$. Observations are conditionally independent given $\mathbf{f}$, with $Y_t$ depending only on $f_t$ through a Gaussian noise model detailed in Section~\ref{sec:gp_likelihood}.

\begin{figure}[t]
    \centering
    \includegraphics[scale=0.5]{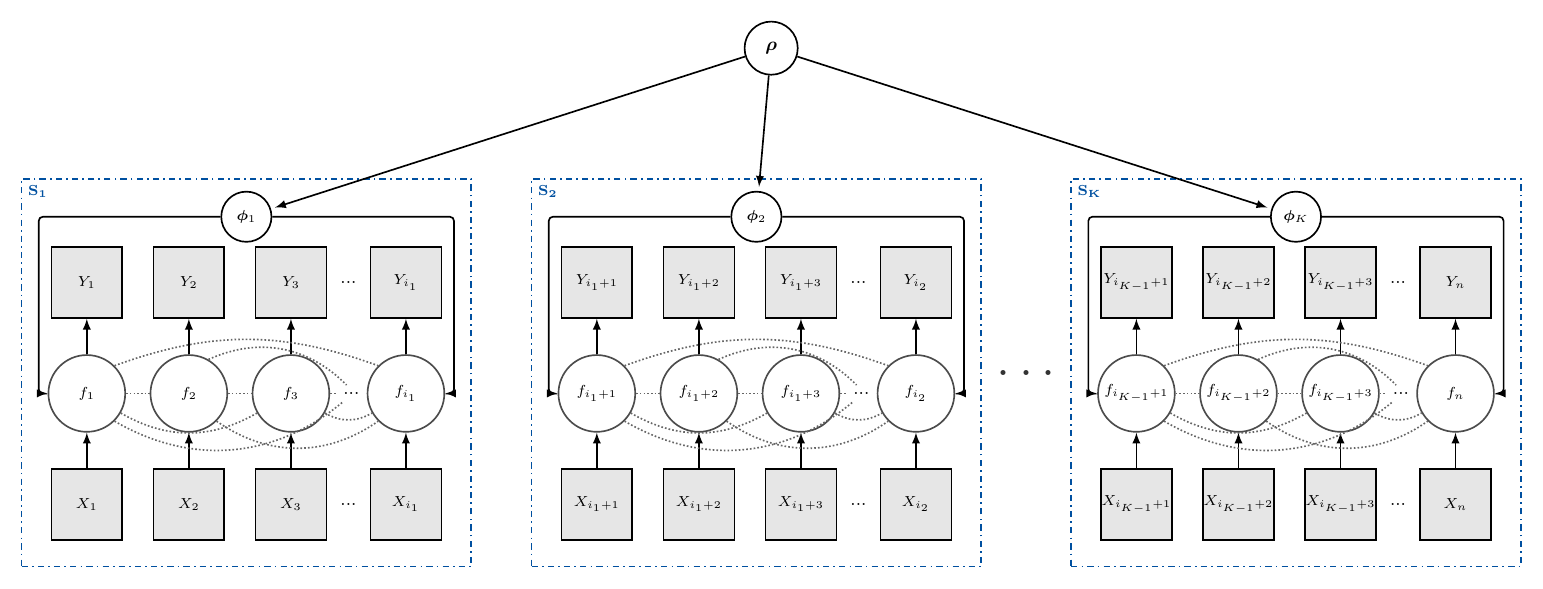}
    \vspace{0.1in}
    \caption{Graphical representation of the TPGPMM. Dashed blue boxes denote conditional independence across regimes, separating the global partition ($\rho$) from local parameters ($\boldsymbol{\phi}_k$). Within each regime, the gray lines connecting latent nodes ($f$) illustrate the non-Markovian GP covariance structure.}
    \label{fig:gpmm_diagram}
\end{figure}

Figure~\ref{fig:gpmm_diagram} shows how dependence flows from the partition $\rho$, which determines the contiguous segments $S_k$, to the regime-specific hyperparameters $\boldsymbol{\phi}_k$, which characterise the non-Markovian GP dependence; combined with the inputs $X_t$, these govern the latent function evaluations $f_t$ that generate the observed sequence $\mathbf{y}$.

\section{Prior on the segmentation}
\label{sec:prior_segmentation}

The prior distribution on the segmentation $\rho$ is specified via a Gibbs-type prior on random partitions, specifically the Pitman--Yor process (PYP) \citep{PitmanYor1997}, which includes the DP as a special case. Rather than fixing $K$, this formulation allows the number and location of regimes to be inferred directly from the data. The partition probability is a product of block terms that depend only on the block sizes, times a weight that depends only on the number of blocks; this is the defining form of a Gibbs-type prior \citep{GnedinPitman2006,DeBlasiEtAl2015}, which induces a coherent joint distribution over the number and lengths of regimes consistent with the underlying EPPF. (Because of the weight on $K$, the prior is not a product partition model in the sense of \citet{BarryHartigan1993}.)

While Gibbs-type priors typically describe unordered exchangeable partitions, temporal order must be preserved in this setting. We therefore consider compositions of $n$, which are in one-to-one correspondence with ordered partitions into contiguous blocks. To obtain a prior on ordered segmentations, we map the exchangeable partition induced by the PYP to an ordered composition using the exchangeable random order construction of \citet{Pitman2006}. This prior closely relates to the one used in \citet{MartinezMena2014}. By strictly reweighting the combinatorial order, this construction preserves the marginal distribution of $K$ while ensuring position neutrality (i.e., conditional on block sizes, all orders are equally likely). This is advantageous because it avoids encoding time-asymmetry and the prior governs only the number and length of segments, while the data likelihood drives the specific placement of changes. As shown in Appendix~\ref{app:pyp-composition-prior}, this results in
\begin{equation}
  \Pr(\rho)
  = \frac{n!}{K!\,\prod_{k=1}^K n_k!}\,
  \frac{\prod_{i=1}^{K-1}(\theta+i\sigma)}{(\theta+1)_{n-1}}
  \prod_{k=1}^K (1-\sigma)_{n_k-1}.
  \label{eq:ordered_py_composition_final}
\end{equation}
Given $\rho$, the corresponding segmentation is obtained by assigning the blocks deterministically to consecutive intervals, thereby enforcing contiguity by construction.

We require each regime to contain at least $m_{\min}$ observations and therefore condition the prior on this minimum-length requirement:
\begin{equation}
  \Pr\nolimits_{m_{\min}}(\rho\mid\sigma,\theta)=\frac{\Pr(\rho\mid\sigma,\theta)\,\mathbf 1\{n_k\ge m_{\min}\ \forall k\}}{Z_n(\sigma,\theta)},
  \label{eq:truncated_prior}
\end{equation}
Here, $Z_n(\sigma,\theta)$ is the prior probability that every regime meets the requirement, computed using the recursion over $K$ in Appendix~\ref{app:pyp-composition-prior}. Because it depends only on $(\sigma,\theta)$, this normalising constant cancels in partition updates when these parameters are held fixed, but must be included when they are updated. Omitting it would implicitly replace the specified hyperprior with one proportional to $p(\sigma)p(\theta)Z_n(\sigma,\theta)$.

We opt to work with the PYP instead of a simpler DP because the DP is the limiting case where $\sigma \to 0$; estimating $\sigma$ allows the data to determine which kind of prior is used, providing a parsimonious yet adaptable prior for segmentation. While the DP implies logarithmic growth in the number of blocks and favours homogeneous segment lengths, the PYP allows for polynomial growth and induces a heavy-tailed distribution where many short segments can coexist with a few long ones.

\section{Gaussian process likelihood within blocks}
\label{sec:gp_likelihood}

Having specified the prior on the regime boundaries, we now turn to the within-regime dependence structure. Conditional on the partition and the induced contiguous blocks, we model the observations within each regime using independent GPs. Each block $k$ is associated with a latent function $f_k$ with a GP prior
\begin{equation}\label{eq:gp_prior_block}
  f_k \mid \boldsymbol{\phi}_k \sim \mathcal{GP}\!\bigl(0,\,C_{\boldsymbol{\phi}_k}(\cdot,\cdot)\bigr),
\end{equation}
and observations within the block follow
\begin{equation}\label{eq:gp_observation_model}
  Y_t \mid f_k(X_t), \sigma^2_{\mathrm{n}}
  \sim \mathcal{N}\!\bigl(f_k(X_t), \sigma^2_{\mathrm{n}}\bigr),
  \qquad t \in S_k,
\end{equation}
where the observation noise variance $\sigma^2_{\mathrm{n}}$ is shared across all blocks and $C_{\boldsymbol{\phi}_k}(\cdot,\cdot)$ denotes a covariance (kernel) function. To prevent the model from explaining regime shifts as noise changes, we share the observation-noise variance across regimes. This way, regimes differ only through their GP hyperparameters.

We adopt a squared exponential, or radial basis function (RBF), kernel, which induces smooth latent functions without assuming parametric Markov transition dynamics:
\begin{equation}\label{eq:rbf_kernel}
  C_{\boldsymbol{\phi}_k}(x,x')
  = s_k^2 \exp\!\left(-\frac{(x-x')^2}{2\ell_k^2}\right),
\end{equation}
with block-specific length-scale $\ell_k$ and signal variance $s_k^2$, so that $\boldsymbol{\phi}_k = (\ell_k, s_k^2)$. Independent log-normal priors are placed on $\ell_k$ and $s_k^2$, with explicit hyperparameters given in Section~\ref{sec:experimental_evaluation}. This flexible kernel allows the GP posterior to capture regime levels nonparametrically, bypassing explicit mean parameters.

For the main specification, we fix the observation-noise standard deviation to help distinguish observation noise from within-regime variation. The GP signal amplitude and the noise level can be difficult to separate from a single time series, so fixing the noise scale reduces ambiguity in the length-scale and amplitude parameters that drive the segmentation. A conservative noise level is intended to absorb short-horizon fluctuations and discourage spurious change-points. The chosen value and its empirical justification are given in Section~\ref{sec:experimental_evaluation}.

To assess the role of the noise assumption and of the within-regime model, we also consider three alternative specifications in which the noise level is estimated rather than fixed. The first keeps the RBF kernel. The second replaces it with a Mat\'ern-3/2 kernel \citep{Stein1999}, which induces continuous but rougher latent functions. The third adds to the Mat\'ern-3/2 model a mean for each regime, which is integrated out analytically; we refer to it as the regime-mean alternative. These specifications allow us to examine how noise estimation, kernel smoothness and an explicit regime mean affect segmentation and prediction. Their kernels, block covariances and priors are given in Appendix~\ref{app:specifications}.

By the conditional-independence assumption~\eqref{eq:phi_factorisation}, the collection $\{(\boldsymbol{\phi}_k, f_k)\}_{k=1}^K$ is independent given $(\rho, \sigma^2_{\mathrm{n}})$. Integrating out $f_k$ yields the marginal likelihood for each block,
\begin{equation}\label{eq:block_marginal_likelihood}
  L_k(\boldsymbol{\phi}_k)
  \;\equiv\;
  p(\mathbf{y}_{S_k} \mid \mathbf{x}_{S_k}, \boldsymbol{\phi}_k, \sigma^2_{\mathrm{n}})
  = \mathcal{N}\!\Bigl(
    \mathbf{y}_{S_k} \,;\,
    \mathbf{0},\,
    K_k(\boldsymbol{\phi}_k) + \sigma^2_{\mathrm{n}} I_{|S_k|}
  \Bigr),
\end{equation}
where $K_k(\boldsymbol{\phi}_k) = \bigl[C_{\boldsymbol{\phi}_k}(X_i, X_j)\bigr]_{i,j \in S_k}$ is the $|S_k|\times|S_k|$ covariance matrix induced by the kernel on the inputs in $S_k$. For any positive semi-definite kernel and $\sigma^2_{\mathrm n}>0$ the matrix $K_k(\boldsymbol{\phi}_k) + \sigma^2_{\mathrm{n}} I_{|S_k|}$ is positive definite, ensuring the likelihood is well-defined. Because the inputs are equally spaced and the kernel is stationary, it is also a symmetric Toeplitz matrix, which the sampler exploits (Section~\ref{sec:inference}). Assuming conditional independence across blocks given the inputs $\mathbf{x}$, the partition $\rho$, the block-specific hyperparameters, and the shared noise variance, the full likelihood factorises as
\begin{equation}\label{eq:full_gp_likelihood_blocks}
  p\!\Bigl(
    \mathbf{y} \mid \mathbf{x}, \rho, \boldsymbol{\phi}_{1:K}, \sigma^2_{\mathrm{n}}
  \Bigr)
  = \prod_{k=1}^K L_k(\boldsymbol{\phi}_k).
\end{equation}

\section{Posterior Computation}
\label{sec:posterior}

Combining the segmentation prior~\eqref{eq:truncated_prior} with the block-wise GP likelihood~\eqref{eq:full_gp_likelihood_blocks}, and with the hyperpriors $p(\boldsymbol{\phi}_k)$, $p(\sigma)$, $p(\theta)$ specified in Section~\ref{sec:experimental_evaluation}, the joint posterior is
\begin{equation}
\pi(\rho, \boldsymbol{\phi}_{1:K}, \sigma, \theta \mid \mathbf{y})
\;\propto\; \Pr\nolimits_{m_{\min}}(\rho \mid \sigma, \theta)\,p(\sigma)\,p(\theta)
\prod_{k=1}^{K}\bigl[\,p(\boldsymbol{\phi}_k)\,L_k(\boldsymbol{\phi}_k)\bigr],
\label{eq:posterior}
\end{equation}
with $L_k$ defined in~\eqref{eq:block_marginal_likelihood}. When $\sigma_{\mathrm n}$ is estimated, it joins the parameters on the left and its prior $p(\sigma_{\mathrm n})$ multiplies the right-hand side. Since the marginal posterior of $\rho$ is analytically intractable, we perform inference jointly on $(\rho, \boldsymbol{\phi}_{1:K}, \sigma, \theta)$ via Markov chain Monte Carlo (MCMC) and obtain marginals for $\rho$ by Monte Carlo.

\subsection{Inference Algorithm}
\label{sec:inference}

Standard MCMC moves that adjust individual change points at random often exhibit poor mixing and high computational cost. We therefore adopt a split--merge strategy that proposes global modifications to the partition structure in a single step. This approach parallels split--merge samplers for mixture models~\citep{JainNeal2004} and their extensions, enabling the chain to escape local modes more effectively. Because the number of blocks $K$ is itself a random variable, structural moves are implemented within a reversible-jump MCMC (RJMCMC) framework~\citep{green1995,RichardsonGreen1997}: a split draws one auxiliary variable per hyperparameter and the merge is its exact inverse, so the move is dimension matched.

The Toeplitz structure makes two further moves cheap. The Cholesky factor of a leading sub-matrix of a covariance is the leading block of its Cholesky factor, and a symmetric Toeplitz matrix is unchanged by reversing the order of its rows and columns. One factorisation of a block's covariance therefore gives the log-likelihood of every prefix and every suffix of the block. We use this to draw each boundary exactly from its full conditional, and to propose split points in proportion to the target given the parent's hyperparameters, an informed proposal in the sense of \citet{Zanella2020}.

Our sampler targets~\eqref{eq:posterior} using a split--merge Metropolis-within-Gibbs scheme. Each iteration updates the partition by a split or merge, draws two regime boundaries from their full conditionals, updates the kernel hyperparameters of three random blocks by random-walk Metropolis, and updates the Pitman--Yor parameters $\sigma$ and $\theta$ (and the noise level, when it is estimated). Because latent functions are integrated out, likelihood evaluations are block-wise and structural proposals recompute only the affected blocks. Proposal details and acceptance ratios are given in Appendix~\ref{app:mcmc_details}, their derivation in Appendix~\ref{app:sampler_derivation}, and the procedure is summarised in Algorithm~\ref{alg:gpmm_inference}.

\begin{algorithm}[htbp]
\caption{TPGPMM split--merge Metropolis-within-Gibbs sampler}
\label{alg:gpmm_inference}
\begin{algorithmic}[1]
\REQUIRE Data $(\mathbf{X},\mathbf{y})$, hyperpriors, total iterations $M$, burn-in $B$, thinning $T$
\STATE Initialise $\rho^{(0)}$, $\boldsymbol{\phi}^{(0)}_{1:K}$, $\sigma^{(0)}$, $\theta^{(0)}$
\FOR{$m=1,\dots,M$}
    \STATE \textbf{(A)} Propose a split (informed split point) or a merge and accept/reject by the RJMCMC ratio in Appendix~\ref{app:mcmc_details}, recomputing $L_k$ only for the affected blocks
    \STATE \textbf{(A$'$)} Draw two randomly chosen boundaries from their full conditionals
    \STATE \textbf{(B)} For three blocks $k$ chosen at random, update $\log\ell_k$ and $\log s_k^2$ by random-walk Metropolis
    \STATE \textbf{(C)} Update $\sigma$ on the logit scale and $\theta$ on the log scale (including $Z_n$); if $\sigma_{\mathrm n}$ is estimated, update $\log\sigma_{\mathrm n}$
    \STATE \textbf{(D)} If $m\bmod 200=0$, refresh all $L_k$ from scratch
    \IF{$m>B$ \AND $(m-B)\bmod T=0$}
        \STATE Store $(\rho^{(m)},\boldsymbol{\phi}^{(m)}_{1:K},\sigma^{(m)},\theta^{(m)})$
    \ENDIF
\ENDFOR
\end{algorithmic}
\end{algorithm}

\subsection{Posterior Summaries}
\label{sec:posterior_summaries}

Let $\{(\rho^{(m)},\boldsymbol{\phi}^{(m)}_{1:K^{(m)}},\sigma^{(m)},\theta^{(m)})\}_{m=1}^{M^\star}$ denote the retained post-burn-in, post-thinning samples. Several non-predictive summaries follow directly from these draws.

We estimate the posterior probability of a regime boundary at index $t$ by
\begin{equation}
\widehat{\tau}_t = \frac{1}{M^\star}\sum_{m=1}^{M^\star}\mathbf{1}\{t\text{ is an internal boundary of }\rho^{(m)}\}.
\end{equation}
Because posterior mass for one boundary is often spread over neighbouring days, we also report the probability of at least one boundary within $\pm10$ trading days of $t$.

To summarise which observations are consistently assigned to the same regime, we construct the posterior similarity matrix $\widehat{S}\in[0,1]^{n\times n}$ with entries $\widehat{S}_{ij} = \frac{1}{M^\star}\sum_m \mathbf{1}\{i\sim_{\rho^{(m)}} j\}$, where $i\sim_\rho j$ denotes that observations $i$ and $j$ belong to the same block under $\rho$.

We also summarise the marginal posteriors of the Pitman--Yor parameters $\sigma$ and $\theta$ using their means, medians and equal-tailed $95\%$ credible intervals computed from the retained draws. The posterior of $K$ is obtained as the empirical distribution of $\{|\rho^{(m)}|\}_{m=1}^{M^\star}$.

The construction of TPGPMM's in-sample posterior predictive density $p(Y_t\mid\mathbf{y})$ and of its forecasts is given in Section~\ref{app:posterior_predictive}.

\section{Experimental setup}
\label{sec:experimental_evaluation}

We evaluate TPGPMM for its two purposes. Section~\ref{sec:regimes} assesses regime identification, on simulated data with known change-points and on two macro-financial series whose regimes we compare with documented events. Section~\ref{sec:prediction} assesses nonstationary prediction, within the sample and beyond its end, on the same two series. The data, baselines, model specifications and computational settings are described below. Appendices~\ref{app:baselines} and \ref{app:metrics} contain implementation and metric details.

The first empirical series is the log-transformed VIX, $y_t = \log(\mathrm{VIX}_t)$, which we use as a proxy for financial stress. The daily observations span January~2015 to December~2023 ($n = 2{,}264$). This timeframe spans low-volatility phases, sudden spikes, and prolonged crises, providing a strong testbed for structural-change detection. We also analyse the daily 10-year US Treasury constant-maturity yield (Federal Reserve Economic Data, FRED, series DGS10) over the same period ($n = 2{,}251$), modelling the standardised yield level (centred and divided by its sample standard deviation). The Treasury yield provides a complementary setting: a trending interest-rate level shaped by monetary-policy cycles, alongside the mean-reverting, stress-sensitive VIX. Both are noisy financial time series. On this scale the fixed noise level $\sigma_{\mathrm n}=0.1$ corresponds to about $9$ basis points of yield. For forecasting, both series are extended with the trading days of 2024, which are held out from fitting.

For regime identification, we compare TPGPMM with Bayesian online change-point detection (BOCPD) with a Normal--Inverse-Gamma run-length model \citep{AdamsMacKay2007}, the sticky hierarchical Dirichlet process switching linear dynamical system (HDP-SLDS) \citep{FoxSudderthJordanWillsky2011}, and the Martinez--Mena method \citep{MartinezMena2014}. We fit Martinez--Mena (MM) under both specifications and with the settings of the original paper (Appendix~\ref{app:martinez_mena}): Model B (MM-B), with the autocorrelation parameter held fixed at zero, treats observations as independent and identically distributed given each block's parameters; Model A (MM-A) allows Markovian autoregressive dependence by assigning the autocorrelation parameter a $\mathrm{U}(0,1)$ prior. For nonstationary prediction, we compare with single stationary GPs with RBF and Mat\'ern-3/2 kernels, fitting their hyperparameters by maximising the marginal likelihood. The fit comparisons additionally include the local approximate GP (laGP) of \citet{GramacyApley2015} and the posterior predictive distributions of HDP-SLDS and Martinez--Mena.

Throughout the experiments, TPGPMM denotes the main specification introduced in Section~\ref{sec:gp_likelihood}, with an RBF kernel and fixed $\sigma_{\mathrm n}=0.1$. The alternatives estimate $\sigma_{\mathrm n}$: TPGPMM ($\sigma_{\mathrm n}$ estimated) keeps the RBF kernel, TPGPMM (Mat\'ern-3/2, $\sigma_{\mathrm n}$ estimated) changes the kernel, and TPGPMM (Mat\'ern-3/2, regime mean, $\sigma_{\mathrm n}$ estimated), the regime-mean alternative, adds an integrated regime mean to the Mat\'ern-3/2 model.

The fixed noise level is informed by first- and second-difference estimates of the short-horizon noise level \citep{Rice1984,GasserEtAl1986}. These estimators are robust to smooth within-regime trends and give standard deviations of $\RiceVix$ and $\GsjVix$ for log-VIX, and $\RiceTr$ and $\GsjTr$ for the standardised Treasury yield. The chosen value of $0.1$ is of the same order of magnitude as these estimates, but deliberately conservative: about twice the estimates for log-VIX and two to three times those for the yield. We use the main specification for regime identification (Section~\ref{sec:regimes}) and compare all four specifications for nonstationary prediction (Section~\ref{sec:prediction}).

For the kernel hyperparameters, we assign independent log-normal priors: $\log \ell_k \sim \mathcal{N}(\log 10,\,0.6^2)$ and $\log s_k^2 \sim \mathcal{N}(\log 1,\,0.8^2)$; for the regime-mean alternative, $\log s_k^2\sim\mathcal N(\log\widehat{\mathrm{var}}(\mathbf y),1)$, and the regime mean has prior mean $\bar y$ and prior variance $\widehat{\mathrm{var}}(\mathbf y)$. The partition follows the Pitman--Yor process in \eqref{eq:truncated_prior} with $\sigma \sim \mathrm{Beta}(1,1)$ and $\theta \sim \mathrm{Gamma}(5,\,0.5)$ (prior mean $10$). When the observation noise is estimated, $\log\sigma_{\mathrm n}\sim\mathcal N(\log0.1,1)$.

Posterior inference uses four chains for each empirical series, each with its own random stream. The chains start from $K_0\in\{1,5,15,30\}$ equal blocks and run for $60{,}000$ iterations on VIX and $20{,}000$ on the Treasury yield. We discard the first 50\% as burn-in and retain every fifth draw. Each simulated dataset is fitted with one chain of $15{,}000$ iterations. Posterior summaries are computed using $M^\star = 800$ retained draws and forecasts using 500. The convergence analysis of the sampler (acceptance rates, effective sample sizes (ESS), and rank-normalised $\widehat{R}$ for $K,\sigma,\theta$) is given in Appendix~\ref{app:sampler_diagnostics}.

The computational cost of inference benefits from conditional independence across blocks, which reduces the cost of GP likelihood evaluation from $\mathcal{O}(n^3)$ to $\mathcal{O}(\sum n_k^3)$. We use Cholesky factorization with $10^{-8}$ diagonal jitter. Each proposal recomputes the likelihoods of the affected blocks, and a full refresh every 200 iterations limits numerical drift. Boundary updates reuse the Cholesky factorisations to evaluate the prefix and suffix likelihoods for all candidate positions. On one core of an Apple Silicon machine, a $60{,}000$-iteration VIX chain ($n=2{,}264$) takes about $\MinPerChainVix$ minutes for TPGPMM and up to $\MinPerChainVixAlt$ minutes for the alternatives (Table~\ref{tab:sampler_diagnostics}); the four chains run in parallel.

\section{Regime identification}
\label{sec:regimes}

This section evaluates the segmentation of TPGPMM, the specification designed for it, and compares it with the change-point baselines. The estimated-noise alternatives are discussed where they segment differently.

\subsection{Simulation Study}
\label{subsec:simulation}

We evaluate change-point detection and the adaptation of the Pitman--Yor discount parameter $\sigma$ on two synthetic scenarios at $n = 900$ with true number of regimes $K = 10$. The homogeneous scenario splits the series into ten segments of equal length 90; the power-law scenario consists of one dominant block of length 810 plus nine shorter blocks of length 10. The scenarios share the same specification and differ only in the segment-size distribution; the exact data-generating configuration is given in Appendix~\ref{app:simulation_config}. Each scenario is run over $50$ replications. For every replication and every method, we record segmentation-accuracy metrics $\hat{K}$, precision, recall, $F_1$, and change-point mean absolute error (CP\_MAE) at a $5$-step tolerance (all defined in Appendix~\ref{app:metrics}).

\begin{table}[htbp]
\centering
\caption{Simulation-study change-point recovery summary across $50$ replications per scenario, for TPGPMM, its three alternatives with estimated noise and the baselines. SD denotes standard deviation.}
\label{tab:simulation-study-recovery}
\resizebox{\linewidth}{!}{%
\begin{tabular}{llccccc}
\toprule
Scenario & Method & $\hat{K}$ (mean $\pm$ SD) & Precision & Recall & F1 & CP\_MAE \\
\midrule
homogeneous & \textbf{TPGPMM} & 10.3 $\pm$ 0.7 & 0.968 & 0.998 & 0.981 & 0.21 $\pm$ 1.41 \\
homogeneous & TPGPMM ($\sigma_{\mathrm n}$ estimated) & 10.3 $\pm$ 0.6 & 0.973 & 0.998 & 0.984 & 0.20 $\pm$ 1.41 \\
homogeneous & TPGPMM (Mat\'ern-3/2, $\sigma_{\mathrm n}$ estimated) & 10.0 $\pm$ 0.2 & 0.996 & 0.996 & \textbf{0.996} & 0.40 $\pm$ 1.98 \\
homogeneous & TPGPMM (Mat\'ern-3/2, regime mean, $\sigma_{\mathrm n}$ estimated) & 10.0 $\pm$ 0.2 & 0.996 & 0.996 & \textbf{0.996} & 0.40 $\pm$ 1.98 \\
homogeneous & BOCPD & 11.1 $\pm$ 1.3 & 0.844 & 0.938 & 0.885 & 4.98 $\pm$ 5.56 \\
homogeneous & HDP-SLDS & 15.5 $\pm$ 4.7 & 0.563 & 0.833 & 0.658 & 8.56 $\pm$ 8.33 \\
homogeneous & MM-A & 9.8 $\pm$ 0.5 & 0.998 & 0.973 & 0.985 & 3.38 $\pm$ 4.73 \\
homogeneous & MM-B & 18.9 $\pm$ 2.6 & 0.493 & 0.964 & 0.650 & 1.84 $\pm$ 2.01 \\
\midrule
power-law & \textbf{TPGPMM} & 10.0 $\pm$ 0.3 & 0.996 & 0.996 & 0.996 & 0.05 $\pm$ 0.22 \\
power-law & TPGPMM ($\sigma_{\mathrm n}$ estimated) & 10.0 $\pm$ 0.2 & 0.996 & 0.998 & \textbf{0.997} & 0.03 $\pm$ 0.16 \\
power-law & TPGPMM (Mat\'ern-3/2, $\sigma_{\mathrm n}$ estimated) & 9.9 $\pm$ 0.3 & 0.998 & 0.989 & 0.993 & 0.11 $\pm$ 0.34 \\
power-law & TPGPMM (Mat\'ern-3/2, regime mean, $\sigma_{\mathrm n}$ estimated) & 9.9 $\pm$ 0.3 & 0.998 & 0.989 & 0.993 & 0.11 $\pm$ 0.34 \\
power-law & BOCPD & 9.4 $\pm$ 1.1 & 0.988 & 0.918 & 0.947 & 2.48 $\pm$ 1.18 \\
power-law & HDP-SLDS & 9.6 $\pm$ 4.2 & 0.614 & 0.567 & 0.554 & 18.92 $\pm$ 64.19 \\
power-law & MM-A & 9.8 $\pm$ 0.4 & 1.000 & 0.976 & 0.987 & 1.22 $\pm$ 0.42 \\
power-law & MM-B & 14.0 $\pm$ 1.8 & 0.675 & 0.958 & 0.788 & 1.39 $\pm$ 0.67 \\
\bottomrule
\end{tabular}}
\end{table}

Table~\ref{tab:simulation-study-recovery} reports TPGPMM, its three alternatives with estimated noise and the baselines. In the power-law scenario TPGPMM achieves an $F_1$ of $\SimFonePow$ and a CP\_MAE of $\SimMaePow$, recovering the true $K = 10$ in $\SimExactPow\%$ of the replications ($\hat{K} = \SimKPow \pm \SimKsdPow$). In the homogeneous scenario it over-segments slightly ($\hat{K} = \SimKHom \pm \SimKsdHom$, exact $K$ in $\SimExactHom\%$ of the replications), so its precision is lower and its $F_1$ ($\SimFoneHom$) is slightly below that of MM-A ($\SimFoneHomMmA$), although its CP\_MAE is lower than that of every baseline.

The alternatives show what estimating the noise changes when the model is well specified. All three recover the noise level (posterior mean $\sigma_{\mathrm n}$ between $\SimSnAlt$ and $\SimSnRbfSn$, truth $0.1$). With the RBF kernel, which is the data-generating kernel, estimating the noise gives results close to TPGPMM's ($F_1$ of $\SimFoneHomRbfSn$ and $\SimFonePowRbfSn$), including the slight over-segmentation of the homogeneous scenario ($\hat K=\SimKHomRbfSn$). The two Mat\'ern-3/2 alternatives are the most accurate in that scenario ($F_1=\SimFoneHomAlt$, exact $K$ in $\SimExactHomAlt\%$ of the replications), even though their kernel differs from the data-generating one. They find the same modal partition in $98$ of the $100$ datasets, so the regime mean makes no difference to segmentation here. On simulated data the noise is therefore identifiable; on the real series it is not, and estimating it with the RBF kernel fragments the segmentation (Section~\ref{subsec:data_vix}).

Among the baselines, MM-A is the closest, with $\hat{K}$ tightly concentrated around the truth ($\SimKHomMmA\pm\SimKsdHomMmA$ and $\SimKPowMmA\pm\SimKsdPowMmA$) and $F_1$ of $\SimFoneHomMmA$ and $\SimFonePowMmA$. It rarely adds a spurious boundary but misses some and localises them less precisely (CP\_MAE $\SimMaeHomMmA$ and $\SimMaePowMmA$). BOCPD performs reasonably on power-law ($F_1 = 0.947$) but over-segments on homogeneous ($F_1 = 0.885$, $\hat{K} = 11.1 \pm 1.3$). MM-B and HDP-SLDS over-segment substantially in both scenarios, with lower precision and inflated CP\_MAE.

The posterior of the Pitman--Yor discount $\sigma$ adapts to the underlying segment-size distribution (Figure~\ref{fig:sigma_posterior_distribution}). In the homogeneous scenario, where the true partition is uniform, the posterior mean is small ($\bar{\sigma} = \SigHom$ on average), close to the DP limit. In the power-law scenario, where one dominant block coexists with many short ones, it moves well away from zero ($\bar{\sigma} = \SigPow$), consistent with a heavier-tailed Pitman--Yor regime.

\label{app:sigma_posterior}

Figure~\ref{fig:sigma_posterior_distribution} compares the posterior mean of $\sigma$ across the $50$ replicate seeds of the simulation study, separately for the two scenarios and for all four specifications. For every specification the two distributions are completely disjoint: across specifications, every homogeneous-scenario seed produces $\overline{\sigma} \in [\SigHomMinAll, \SigHomMaxAll]$ while every power-law seed produces $\overline{\sigma} \in [\SigPowMinAll, \SigPowMaxAll]$, and the four specifications give almost the same values. This per-seed separation provides direct evidence that the posterior on $\sigma$ adapts to the segment-size distribution of the data---staying near the Dirichlet-process limit when blocks are uniform and moving distinctly above zero when block sizes are heavy-tailed---rather than being pinned by the prior. Because it does not depend on the kernel, the noise model or the regime mean, the adaptation is a property of the partition prior. Including the truncation constant $Z_n$ in~\eqref{eq:truncated_prior} matters here: without it the effective prior on $\sigma$ has mean $0.04$ at $n=900$, which pulls both scenarios towards the Dirichlet-process limit.

\begin{figure}[ht]
    \centering
    \includegraphics[width=\linewidth]{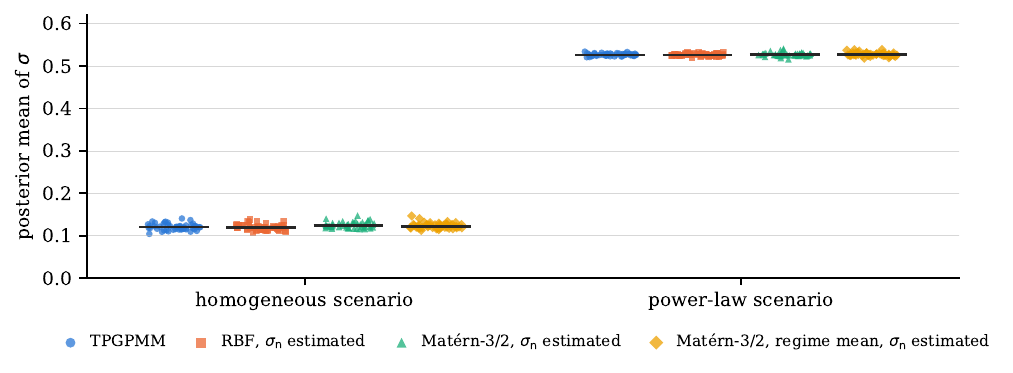}
    \caption{Distribution of the posterior mean of $\sigma$ across $50$ replicate seeds in the homogeneous and power-law scenarios, for TPGPMM and its three alternatives with estimated noise (one colour and marker per specification). Each point is the posterior mean for one seed, jittered within each specification; horizontal lines mark the medians.}
    \label{fig:sigma_posterior_distribution}
\end{figure}

These results highlight the role of two key ingredients in the TPGPMM architecture: a global Pitman--Yor-induced prior on contiguous compositions, which concentrates posterior mass on well-separated boundaries, and a within-block likelihood flexible enough to absorb regime-internal variation without creating spurious change points. In the synthetic setting, MM-A's Ornstein--Uhlenbeck process is expressive enough to generate good performance, whereas BOCPD and HDP-SLDS rely on local Markovian or run-length dynamics, and MM-B retains the partition prior but removes within-block dependence.

\subsection{VIX}
\label{subsec:data_vix}

We benchmark the posterior segmentation of log-VIX against six major market events, listed in Table~\ref{tab:vix_events}, and use them as economic references to interpret the inferred regimes rather than as model constraints.

\begin{table}[ht]
\centering
\caption{Six benchmark market events used to interpret the posterior VIX segmentation.}
\small
\begin{tabularx}{\textwidth}{l X}
\toprule
\textbf{Event (Date)} & \textbf{Market Impact} \\
\midrule
\textbf{China ``Black Monday''} (2015-08-24) & China slowdown triggered U.S.\ liquidity strains. \\
\textbf{Brexit Vote} (2016-06-24) & Surprise ``Leave'' vote sparked global equity sell-offs. \\
\textbf{``Volmageddon''} (2018-02-05) & VIX spike driven by short-volatility deleveraging. \\
\textbf{COVID Panic} (2020-03-16) & Pandemic lockdowns caused acute funding strains. \\
\textbf{Russia--Ukraine} (2022-02-24) & Invasion led to risk sell-offs and commodity surges. \\
\textbf{US Bank Turmoil} (2023-03-10) & Silicon Valley Bank failure sparked bank solvency and rate volatility. \\
\bottomrule
\end{tabularx}
\label{tab:vix_events}
\end{table}

\begin{figure}[t!]
    \centering
    \includegraphics[width=\linewidth]{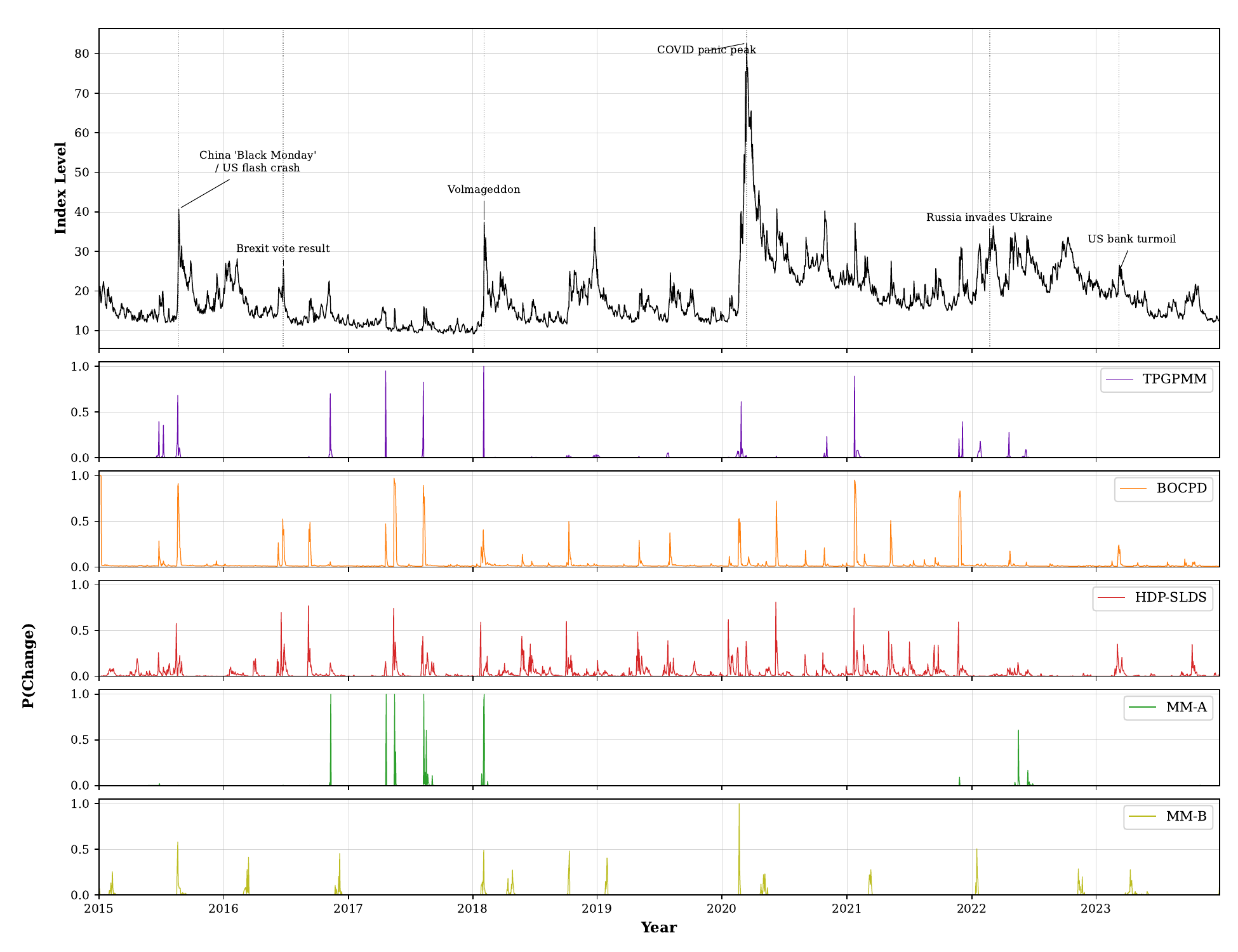}
    \caption{Analysis of the VIX dataset. Top panel: volatility index with the six benchmark market events. Lower five panels: posterior change-point probabilities $\widehat{\tau}_t$ for TPGPMM and the benchmark models.}
    \label{fig:vix_analysis_full}
\end{figure}

Figure~\ref{fig:vix_analysis_full} separates the five methods along two axes: the confidence of the change-point detections and the sparsity of the inferred segmentation. TPGPMM produces few, confident detections: its change-point probability exceeds $0.5$ at $\PeaksTpHalf$ locations and $0.2$ at $\PeaksTpFifth$. MM-A is as sparse ($\PeaksMmAHalf$ peaks above $0.5$ and $\MmKVixA$ regimes on average), but most of its boundaries fall between November 2016 and February 2018: it isolates the short spikes of August 2017 and February 2018 and places no boundary at the China sell-off ($\EvMmAVixA$) or the COVID-19 panic ($\EvMmAVixD$). MM-B has only $\PeaksMmBHalf$ peaks above $0.5$ but spreads moderate mass over many locations ($\PeaksMmBFifth$ above $0.2$). BOCPD and HDP-SLDS produce denser and more diffuse boundary probabilities ($\PeaksBocpdFifth$ and $\PeaksHdpFifth$ peaks above $0.2$), with HDP-SLDS in particular yielding a near-continuous stream of low-confidence transition candidates.

The posterior of TPGPMM places a mean of $\VixK$ regimes on the VIX (95\% interval $[\VixKlo,\VixKhi]$), and its modal partition has $K=\VixModalK$ regimes, with change-points on \VixCpA, \VixCpB, \VixCpC, \VixCpD, \VixCpE, \VixCpF, \VixCpG, \VixCpH, \VixCpI{} and \VixCpJ. The probability of a boundary within $\pm10$ trading days of these dates is $\VixCpPA$, $\VixCpPB$, $\VixCpPC$, $\VixCpPD$, $\VixCpPE$, $\VixCpPF$, $\VixCpPG$, $\VixCpPH$, $\VixCpPI$ and $\VixCpPJ$. Three of the six benchmark events are matched:
\begin{itemize}
\item \textbf{China, August 2015} ($\EvVixA$): a new regime opens on 21 August, the first day of the sell-off that culminated on 24 August.
\item \textbf{Volmageddon, February 2018} ($\EvVixC$): a new regime opens on 5 February 2018.
\item \textbf{COVID-19, March 2020} ($\EvVixD$): a new regime opens on 2 March 2020, at the start of the sell-off, two weeks before the 16 March peak.
\end{itemize}
The 2016 Brexit vote ($\EvVixB$), the Russia--Ukraine invasion ($\EvVixE$) and the 2023 bank turmoil ($\EvVixF$) receive no boundary mass: these short-lived spikes are absorbed as within-regime variation of regimes that began earlier. TPGPMM thus identifies persistent structural shifts rather than isolated excursions, at the cost of not marking every major event.

The alternatives with estimated noise behave differently. With the RBF kernel and estimated noise ($\sigma_{\mathrm n}=\VixSnRbf$), the posterior fragments the series into $\VixKRbfSn$ regimes on average. With the rougher Mat\'ern-3/2 kernel ($\sigma_{\mathrm n}=\VixSnAlt$), the posterior is again parsimonious ($\VixKAlt$ regimes) and matches China and Volmageddon ($\EvAltVixA$ and $\EvAltVixC$), but it absorbs the COVID-19 panic ($\EvAltVixD$) into a long, elevated-volatility regime that runs from February 2018 to June 2022.

The BOCPD posterior at index $t$ is a distribution over the run-length (time since the last change point) and tends to spread mass over windows of plausible boundary locations. Similarly, the Sticky HDP-SLDS uses a Markovian regime-switching prior on latent states but lacks global partition coherence, which can produce local and noisy transition probabilities.

Martinez--Mena also employs a Pitman--Yor-induced prior on ordered compositions; consequently, the differences between MM and TPGPMM are driven by the within-block likelihood rather than the prior. In the log-VIX application, the autocorrelation parameter of MM-A saturates near the unit root (posterior mean $\phi=\MmPhiVix$), so each regime behaves like a random walk whose accumulated drift absorbs persistent level shifts: MM-A merges the China sell-off and COVID-19 into long regimes while still isolating a few short spikes. MM-B, whose observations are independent within a block, reacts instead to transient level changes and uses $\MmKVixB$ regimes on average. TPGPMM, by contrast, models within-block dependence as a GP with non-Markovian covariance. Transient spikes can be absorbed into the within-block GP posterior, and the marginal likelihood penalizes adding new hyperparameters $(\ell_k, s_k^2)$ unless persistence and amplitude genuinely differ across a candidate boundary. This mechanism explains why TPGPMM's segmentation in Figure~\ref{fig:vix_analysis_full} is both sparse and aligned with persistent episodes.

\begin{figure}[t]
    \centering
    \includegraphics[width=\linewidth]{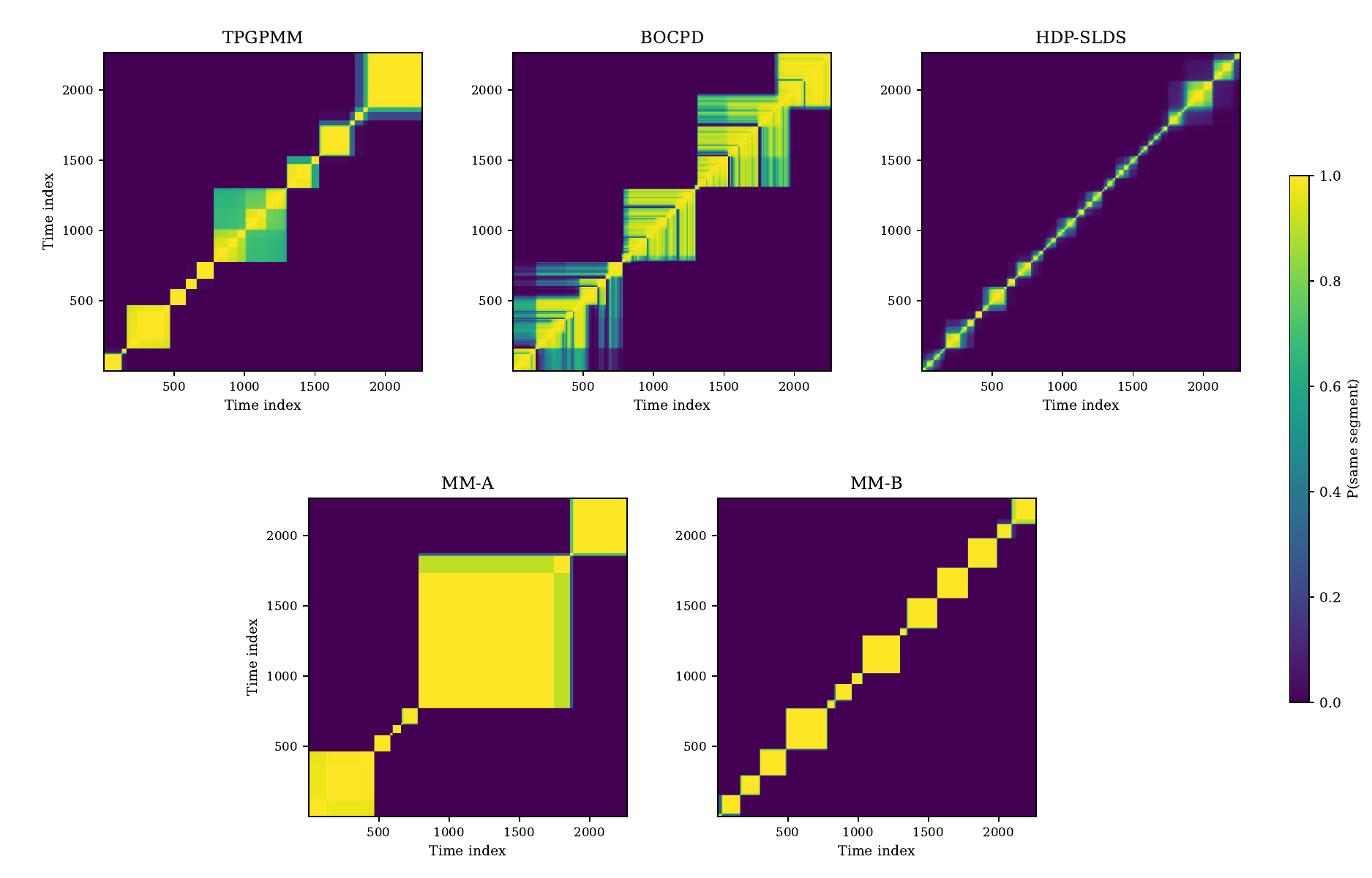}
    \caption{Posterior similarity matrices for the VIX time series under each method.}
    \label{fig:vix_psm}
\end{figure}

To assess partition stability, we compute the Posterior Similarity Matrix (PSM) for each method, where entry $(i,j)$ is the posterior probability that observations $i$ and $j$ belong to the same regime (Figure~\ref{fig:vix_psm}). TPGPMM yields a clean block-diagonal structure. Eight of its ten modal boundaries are sharp (probability $0.96$ or more). The main uncertainty concerns a possible split of the 2018--2020 regime during 2019 (boundary probabilities $0.21$--$0.28$) and the boundaries of late 2021 and 2022, which receive probabilities between $0.36$ and $0.68$. MM-A and MM-B also yield sharply block-diagonal PSMs. MM-A's is dominated by two long blocks, from 2015 to November 2016 and from February 2018 to May 2022, which merge the China sell-off and the Brexit vote, and the COVID-19 panic and the Russia--Ukraine invasion; MM-B's has many smaller blocks. BOCPD exhibits cross-block bleeding consistent with its diffuse run-length posterior, while HDP-SLDS collapses to a near-diagonal pattern indicative of many short, weakly distinguished regimes.

Across the five methods, only TPGPMM produces a segmentation that is at once sparse, confident and aligned with persistent market episodes. The results suggest that parsimonious change-point detection requires both a global prior on contiguous partitions and a within-block likelihood flexible enough to absorb transient variation without introducing unnecessary regimes.

\subsubsection{Interpretation of the inferred regimes}
\label{app:regime_interpretation}

To obtain an interpretable summary of the latent structure, we extract the \emph{modal partition} $\widehat{\rho}$, defined as the most frequently sampled composition $\rho^{(m)}$ in the MCMC output. It accounts for only $\VixModalShare\%$ of the draws, because boundary positions vary by a few days between draws; the boundary probabilities reported above show that the boundaries themselves are well identified. Conditional on this fixed partition, each block's kernel hyperparameters $\boldsymbol{\phi}_k = (\ell_k, s_k^2)$ are set to the posterior mean of $\{\boldsymbol{\phi}_k^{(m)}\}$ over the subset of MCMC draws with $\rho^{(m)} = \widehat{\rho}$, and the per-block posterior mean function is then computed by standard GP conditioning on the within-block data with the noise variance $\sigma^2_{\mathrm n}$. We refer to the resulting fit as the \emph{maximum a posteriori} (MAP) segmentation for brevity, although it is technically a partition-MAP with within-block posterior-mean kernels.

Figure~\ref{fig:vix_map} shows the MAP segmentation for the VIX series. The eleven inferred regimes recover an economically coherent narrative; the dates of the boundaries coincide with the events named below, which we use only as interpretation:
\begin{itemize}
\item \textbf{Block 1} is the calm first half of 2015 (VIX around $15$), with a long length-scale.
\item \textbf{Block 2} opens on 29 June 2015, when Greece closed its banks, and runs to the eve of the August sell-off.
\item \textbf{Block 3} runs from the China sell-off of August 2015 to the 2016 US election. It is a volatile phase with a short length-scale ($\ell=11.7$) that absorbs the Brexit vote.
\item \textbf{Blocks 4--6} capture the exceptionally calm period from November 2016 to January 2018, with the lowest VIX levels in the sample (around $10$--$12$) and long length-scales. Their boundaries fall in late April 2017, after the first round of the French presidential election, and on 10 August 2017, a short-lived spike.
\item \textbf{Block 7} opens with Volmageddon on 5 February 2018 and runs to the end of February 2020, absorbing the late-2018 sell-off.
\item \textbf{Block 8} is the COVID-19 regime, from 2 March 2020 to January 2021, with the highest level in the sample (VIX around $30$).
\item \textbf{Block 9} opens with the retail-trading volatility spike of late January 2021 and has the largest signal variance ($s^2=5.2$).
\item \textbf{Blocks 10 and 11} cover December 2021 to April 2022, which contains the Russia--Ukraine invasion, and the 2022 bear market and its normalisation through 2023, which contains the bank turmoil of March 2023. Their boundaries are the least certain ($\VixCpPI$ and $\VixCpPJ$).
\end{itemize}

Throughout this discussion, ``high-volatility'' and ``low-volatility'' regimes refer to the level of $\log(\mathrm{VIX}_t)$ implied by the
within-regime posterior mean of $f_k$, not to a model parameter: TPGPMM
characterises each regime by its kernel hyperparameters $(\ell_k,s_k^2)$---persistence and amplitude---and the level itself emerges from conditioning on the within-block data.

\begin{figure}[t]
    \centering
    \includegraphics[width=\linewidth]{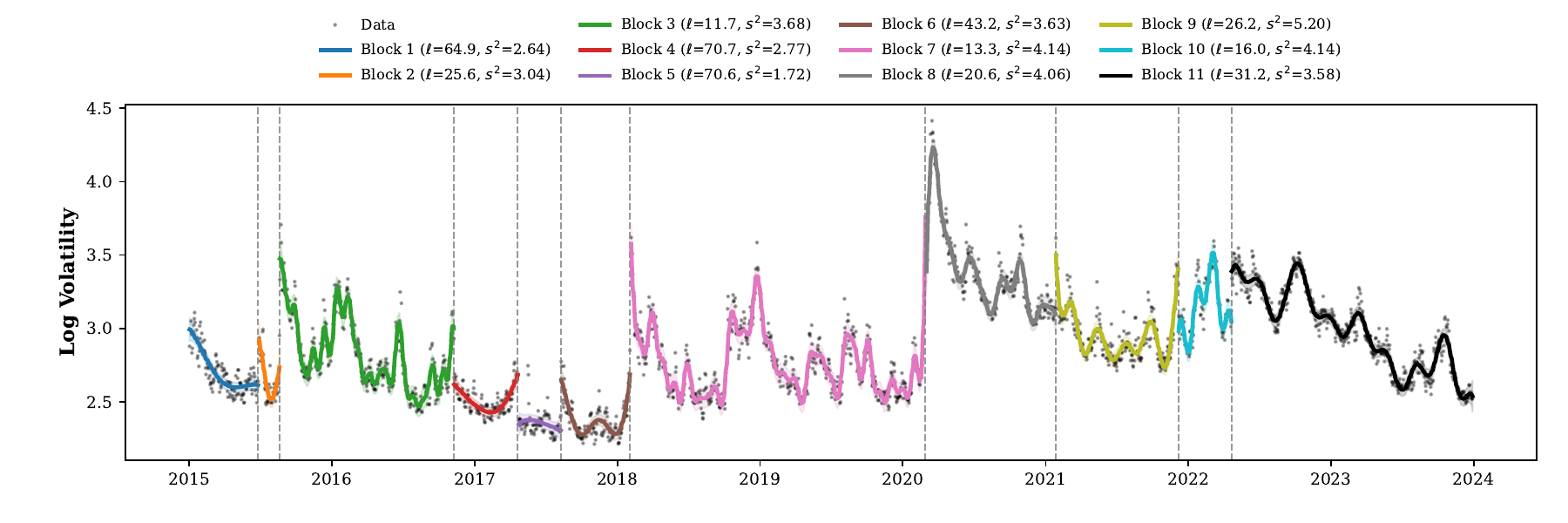}
    \caption{MAP segmentation of the VIX series. Coloured lines show the per-block GP posterior means under the modal partition $\widehat{\rho}$, with kernel hyperparameters fixed to the posterior mean across MCMC draws matching that partition. Dashed vertical lines mark the inferred internal change-points. The legend gives the averaged kernel hyperparameters $(\ell_k, s_k^2)$ for each of the eleven blocks.}
    \label{fig:vix_map}
\end{figure}

\subsection{10-year Treasury yield}
\label{subsec:data_treasury}
\label{app:treasury}

The posterior of TPGPMM places a mean of $\TrK$ regimes on the standardised yield (95\% interval $[\TrKlo,\TrKhi]$), with a $\TrModalK$-regime MAP partition. We interpret it against three monetary-policy events, listed in Table~\ref{tab:treasury_events}.

\begin{table}[ht]
\centering
\caption{Benchmark monetary-policy events used to interpret the posterior
10-year Treasury segmentation.}
\small
\begin{tabularx}{\textwidth}{l X}
\toprule
\textbf{Event (Date)} & \textbf{Policy / market impact} \\
\midrule
\textbf{Fed mid-cycle easing} (2019-07-31) & First of three 2019 rate cuts as growth and trade-war risks mounted. \\
\textbf{COVID emergency cuts} (2020-03-15) & Policy rate cut to the zero lower bound; large-scale asset purchases. \\
\textbf{Hawkish pivot} (2021-12-15) & The Federal Open Market Committee (FOMC) accelerated tapering, signalling the 2022 tightening cycle. \\
\bottomrule
\end{tabularx}
\label{tab:treasury_events}
\end{table}

The MAP partition (Table~\ref{tab:treasury_segmentation}) splits the sample into five economically coherent regimes:
\begin{itemize}
\item \textbf{Normalisation and hiking cycle,} from 2015 to mid-2019. It covers the December 2015 lift-off, the 2016--2018 rate hikes and the balance-sheet runoff.
\item \textbf{Mid-cycle easing,} from mid-2019. The timing of this boundary is uncertain: the modal partition places it on 30 July, the day before the first 2019 cut ($\EvTrA$ within $\pm10$ days), but more posterior mass ($0.69$) falls in early September 2019.
\item \textbf{COVID and quantitative easing,} from 20 March 2020, five days after the emergency cuts of 15 March, with probability $\EvTrB$.
\item \textbf{Tightening onset,} from late February 2022, three weeks before the first hike (probability $0.47$).
\item \textbf{Restrictive policy,} from mid-June 2022, when the Fed moved to 75-basis-point hikes. The mass for this boundary is spread over June and July 2022.
\end{itemize}
The COVID cuts are identified with certainty, and the 2019 easing and the 2022 tightening are detected, but with uncertain timing. The hawkish pivot of December 2021 receives little mass of its own ($\EvTrC$). The Mat\'ern-3/2 alternative with estimated noise dates the tightening more sharply (probability $0.93$ in February 2022) but gives no boundary for the 2019 easing ($\EvAltTrA$). As with VIX, transient moves that do not persist are absorbed within regimes.

\begin{table}[ht]
\centering
\small
\caption{MAP segmentation of the standardised 10-year Treasury yield under TPGPMM. Mean yield (in percent) and daily-change standard deviation (in basis points) are computed within each inferred block; the boundary posterior probability is the posterior probability of a change-point within $\pm10$ trading days of the block's left boundary (regime~1 has no left boundary).}
\label{tab:treasury_segmentation}
\resizebox{\linewidth}{!}{%
\begin{tabular}{c>{\raggedright\arraybackslash}p{3.2cm}lrrrc}
\toprule
Regime & Interpretation & Dates & $N$ & Mean yield & Daily-change SD & Boundary post.\ prob. \\
\midrule
$1$ & Normalisation and hiking cycle & 2015-01-02 to 2019-07-29 & $1{,}144$ & $2.32\%$ & $4.22$ bp & --- \\
$2$ & Mid-cycle easing & 2019-07-30 to 2020-03-19 & $160$ & $1.65\%$ & $6.50$ bp & $0.22$ \\
$3$ & COVID; quantitative easing & 2020-03-20 to 2022-02-25 & $486$ & $1.19\%$ & $4.26$ bp & $1.00$ \\
$4$ & Tightening onset & 2022-02-28 to 2022-06-10 & $73$ & $2.61\%$ & $7.15$ bp & $0.47$ \\
$5$ & Restrictive policy & 2022-06-13 to 2023-12-29 & $388$ & $3.77\%$ & $7.74$ bp & $0.38$ \\
\bottomrule
\end{tabular}}
\end{table}

\section{Nonstationary prediction}
\label{sec:prediction}

This section evaluates TPGPMM as a nonstationary predictive model and asks whether it predicts better than a stationary GP. We compare the posterior predictive distributions of all four specifications with stationary GPs whose hyperparameters are fitted to the data: their in-sample fit on 2015--2023, and their forecasts over the held-out year 2024. With a regime mean and estimated noise, the nonstationary model fits the training series more closely than the stationary GPs and has lower forecast RMSE except at the shortest horizons. Here the estimated-noise alternatives are the relevant models, and TPGPMM shows what its conservative noise level costs.

The results reveal a trade-off between sparse segmentation and predictive accuracy. TPGPMM produces the sparse, confident segmentation described in Section~\ref{sec:regimes}, with $\VixK$ regimes on VIX and $\TrK$ on the Treasury yield on average. Estimating the noise improves prediction. Every alternative fits more closely than TPGPMM in sample on both series; on VIX, their CRPS/sd is $\InsRbfSnVixCrps$, $\InsAltVixCrps$ and $\InsFlexVixCrps$, against $\InsMainVixCrps$. The alternatives also forecast more accurately in most comparisons across horizons and metrics, and the regime-mean alternative in all of them. These predictive gains are accompanied by changes in the inferred segmentation. With the RBF kernel, estimating the noise fragments the series into $\VixKRbfSn$ regimes on VIX and $\TrKRbfSn$ on the Treasury yield, and the regime-mean alternative uses $\FlexKvix$ and $\FlexKtr$. The Mat\'ern-3/2 alternative remains parsimonious ($\VixKAlt$ and $\TrKAlt$ regimes) but misses the COVID-19 regime on VIX and the 2019 easing on the Treasury yield.

\subsection{Posterior predictive distribution}
\label{app:posterior_predictive}

To characterise the latent function through a mixture of GPs and quantify
in-sample uncertainty, we approximate the in-sample posterior predictive
distribution via Monte Carlo integration, averaging over the MCMC draws
of the partition and GP hyperparameters. The summaries reported here
evaluate $p(Y_t \mid \mathbf{y})$ at the same inputs $X_t$ used to fit
the model; standard GP conditioning extends the construction to any
$x^\star$ at no additional cost.

For a given MCMC sample $m$, observation $Y_t$ falls within some block
$k$ under $\rho^{(m)}$, and the conditional predictive distribution is
Gaussian by standard GP conditioning on the within-block data:
\begin{equation}
  p\!\bigl(Y_t \mid \rho^{(m)},\, \boldsymbol{\phi}_k^{(m)},\,
            \sigma^{2(m)}_{\mathrm{n}},\, \mathbf{y}\bigr)
  = \mathcal{N}\!\bigl(Y_t \mid \mu_t^{(m)}, \nu_t^{(m)}\bigr),
\end{equation}
with
\begin{align}
  \mu_t^{(m)} &= (\mathbf{k}_{*}^{(m)})^{\top}
                 \bigl(K_k^{(m)} + \sigma^{2(m)}_{\mathrm{n}}\,I\bigr)^{-1}
                 \mathbf{y}_{S_k}, \\
  \nu_t^{(m)} &= k_{**}^{(m)}
               - (\mathbf{k}_{*}^{(m)})^{\top}
                 \bigl(K_k^{(m)} + \sigma^{2(m)}_{\mathrm{n}}\,I\bigr)^{-1}
                 \mathbf{k}_{*}^{(m)}
               + \sigma^{2(m)}_{\mathrm{n}},
\end{align}
where $K_k^{(m)}$ is the kernel matrix for block $k$ at the training
inputs, $\mathbf{k}_{*}^{(m)}$ is the cross-covariance vector between
$X_t$ and the within-block training inputs, and $k_{**}^{(m)}$ is the
prior variance at $X_t$, all computed using $\boldsymbol{\phi}_k^{(m)}$.
Marginalising over the joint posterior of
$(\rho,\boldsymbol{\phi}_{1:K},\sigma,\theta,\sigma_{\mathrm n})$ yields the mixture
\begin{equation}
  p(Y_t \mid \mathbf{y}) \;\approx\;
  \frac{1}{M^\star} \sum_{m=1}^{M^\star}
  \mathcal{N}\!\bigl(Y_t \mid \mu_t^{(m)}, \nu_t^{(m)}\bigr).
  \label{eq:posterior_predictive_mixture}
\end{equation}
The predictive mean is the average of the $\mu_t^{(m)}$ and credible intervals are the exact quantiles of the mixture \eqref{eq:posterior_predictive_mixture}. This naturally propagates uncertainty in the partition, the within-regime GP hyperparameters and, when it is estimated, the noise level, and widens the bands in regions of structural ambiguity. Forecasts are built the same way: each draw forecasts from its own final regime, and the draws are averaged.

\begin{figure}[t]
    \centering
    \includegraphics[width=\linewidth]{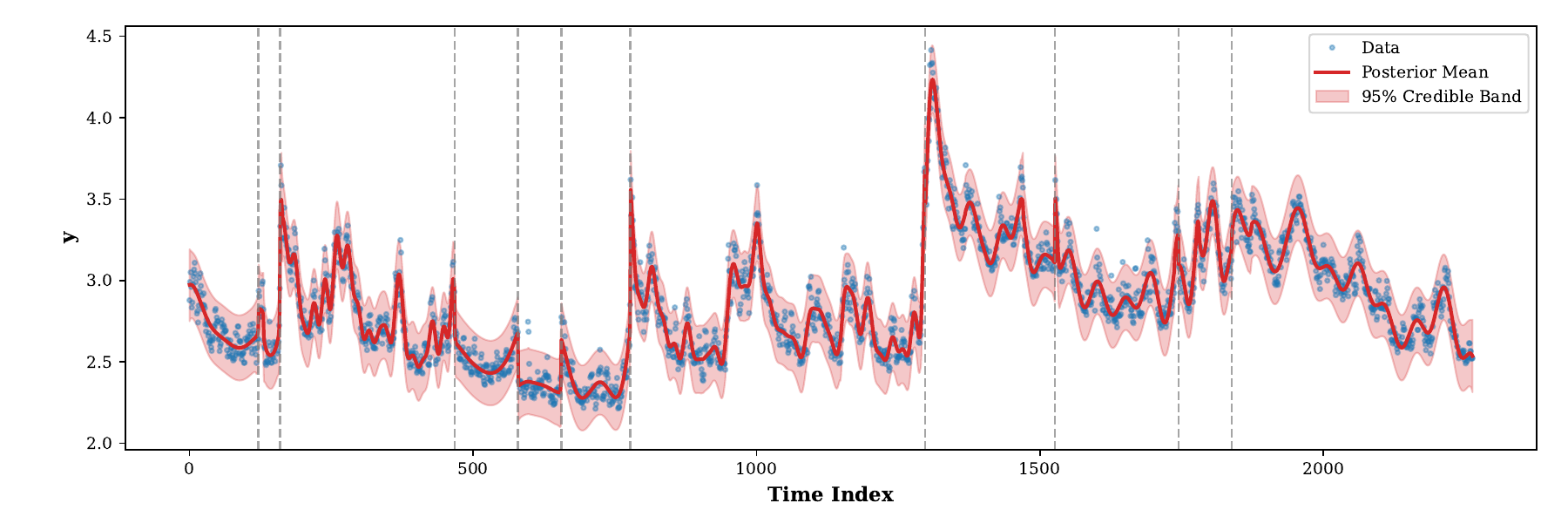}
    \caption{Posterior predictive mixture under TPGPMM on log-VIX. The red
    line is the posterior mean of $Y_t$, the shaded band is the $95\%$
    credible band, and the dashed vertical lines mark the modal-partition
    change-points. The construction follows
    Equation~\eqref{eq:posterior_predictive_mixture}.}
    \label{fig:vix_predictive}
\end{figure}

Figure~\ref{fig:vix_predictive} shows the posterior predictive mixture
on the log-VIX series. The red curve is the posterior mean of $Y_t$,
the shaded band is the pointwise $95\%$ credible band, and the dashed
vertical lines mark the modal-partition change-points discussed in
Section~\ref{subsec:data_vix}. The fitted predictive mean tracks the
series closely within each regime, including during the COVID-19 peak.

\subsection{In-sample fit}
\label{subsec:fit}

Table~\ref{tab:vix_predictive_comparison} reports the \emph{in-sample} posterior-predictive fit on the log-VIX series. On the scale-normalised metrics, which are the comparable ones because the baselines operate on different targets, TPGPMM (RMSE/sd $\InsMainVixRmse$, CRPS/sd $\InsMainVixCrps$) fits better than laGP, HDP-SLDS and MM-B, about as well as MM-A ($\InsMmAVixRmse$, $\InsMmAVixCrps$), and less closely than the fitted stationary GPs (RBF: $\InsGprVixRmse$, $\InsGprVixCrps$; Mat\'ern-3/2: $\InsGpmVixRmse$, $\InsGpmVixCrps$), because its conservative noise level deliberately attributes daily fluctuations to noise. MM-A has no observation noise, so its predictive of $y_t$ conditions on $y_{t-1}$ within the regime rather than on $y_t$ (Appendix~\ref{app:martinez_mena}). The alternatives with estimated noise fit more closely than all the change-point baselines, and the regime-mean alternative fits best of all ($\InsFlexVixRmse$, $\InsFlexVixCrps$). The fixed noise level that stabilises TPGPMM's segmentation therefore costs predictive accuracy, and we use TPGPMM for segmentation and the regime-mean alternative for prediction. Because in-sample fit rewards flexibility, the out-of-sample comparison is the forecasting exercise of Section~\ref{subsec:forecast}.

\begin{table}[ht]
\centering
\small
\caption{\emph{In-sample} posterior-predictive fit on the log-VIX series ($n\approx2264$ trading days). RMSE, NLPD and CRPS are the in-sample posterior-predictive metrics; RMSE/sd, $\Delta$NLPD and CRPS/sd are their scale-normalised variants. Because the baselines operate on different targets (second column), the normalised columns are the comparable ones. Lower is better throughout. The standard GPs have hyperparameters fitted by maximum marginal likelihood; the laGP and HDP-SLDS rows are from the original study, and Mart\'inez--Mena A and B are fitted with the settings of \citet{MartinezMena2014}. Mart\'inez--Mena A has no observation noise, so its predictive of $y_t$ conditions on $y_{t-1}$ within the regime (one step ahead) rather than on $y_t$ itself.}
\label{tab:vix_predictive_comparison}
\resizebox{\linewidth}{!}{%
\begin{tabular}{llrrrrrrr}
\toprule
Method & Target & RMSE & NLPD & CRPS & RMSE/sd & $\Delta$NLPD & CRPS/sd & $n$ \\
\midrule
\textbf{TPGPMM} & $\log$ VIX & $0.0714$ & $-1.1042$ & $0.0417$ & $0.2082$ & $-1.4537$ & $0.1215$ & $2264$ \\
TPGPMM ($\sigma_{\mathrm n}$ estimated) & $\log$ VIX & $0.0564$ & $-1.4339$ & $0.0318$ & $0.1642$ & $-1.7834$ & $0.0928$ & $2264$ \\
TPGPMM (Mat\'ern-3/2, $\sigma_{\mathrm n}$ estimated) & $\log$ VIX & $0.0409$ & $-1.7048$ & $0.0233$ & $0.1193$ & $-2.0543$ & $0.0678$ & $2264$ \\
TPGPMM (Mat\'ern-3/2, regime mean, $\sigma_{\mathrm n}$ estimated) & $\log$ VIX & $0.0261$ & $-2.0406$ & $0.0159$ & $0.0760$ & $-2.3901$ & $0.0463$ & $2264$ \\
\midrule
Standard GP (RBF) & $\log$ VIX & $0.0502$ & $-1.5255$ & $0.0277$ & $0.1464$ & $-1.8750$ & $0.0808$ & $2264$ \\
Standard GP (Mat\'ern-3/2) & $\log$ VIX & $0.0394$ & $-1.7117$ & $0.0223$ & $0.1149$ & $-2.0612$ & $0.0651$ & $2264$ \\
\midrule
laGP & $\log$ VIX & $0.1328$ & $-0.7501$ & $0.0691$ & $0.3869$ & $-1.0996$ & $0.2013$ & $2264$ \\
HDP-SLDS & $\log(r_t^2+\epsilon)$ & $0.7444$ & $1.2443$ & $0.3903$ & $0.2900$ & $-1.1174$ & $0.1520$ & $2263$ \\
Mart\'inez--Mena A & $\log$ VIX & $0.0769$ & $-1.2057$ & $0.0414$ & $0.2241$ & $-1.5552$ & $0.1206$ & $2264$ \\
Mart\'inez--Mena B & $\log$ VIX & $0.1441$ & $-0.4771$ & $0.0820$ & $0.4199$ & $-0.8266$ & $0.2388$ & $2264$ \\
\bottomrule
\end{tabular}}
\end{table}

On the Treasury yield, Table~\ref{tab:treasury_predictive_comparison} reports the in-sample posterior-predictive fit for the four specifications and the baselines. On the comparable scale-normalised
metrics, TPGPMM (RMSE/sd $\InsMainTrRmse$, CRPS/sd $\InsMainTrCrps$) fits better than laGP ($0.092$, $0.046$), HDP-SLDS and MM-B but less closely than MM-A ($\InsMmATrRmse$, $\InsMmATrCrps$), whose regimes, with $\phi=\MmPhiTr$, track the yield from one day to the next, and than the fitted stationary GPs (CRPS/sd $\InsGprTrCrps$ for RBF and $\InsGpmTrCrps$ for Mat\'ern-3/2). The alternatives with estimated noise fit more closely than all the change-point baselines, and the regime-mean alternative fits best ($\InsFlexTrRmse$, $\InsFlexTrCrps$).

\begin{table}[ht]
\centering
\small
\caption{\emph{In-sample} posterior-predictive fit on the standardised 10-year Treasury yield ($n\approx2{,}251$). Columns as in Table~\ref{tab:vix_predictive_comparison}.}
\label{tab:treasury_predictive_comparison}
\resizebox{\linewidth}{!}{%
\begin{tabular}{llrrrrrrr}
\toprule
Method & Target & RMSE & NLPD & CRPS & RMSE/sd & $\Delta$NLPD & CRPS/sd & $n$ \\
\midrule
\textbf{TPGPMM} & std.\ yield & $0.0605$ & $-1.1785$ & $0.0371$ & $0.0605$ & $-2.5972$ & $0.0371$ & $2251$ \\
TPGPMM ($\sigma_{\mathrm n}$ estimated) & std.\ yield & $0.0357$ & $-1.8700$ & $0.0203$ & $0.0357$ & $-3.2887$ & $0.0203$ & $2251$ \\
TPGPMM (Mat\'ern-3/2, $\sigma_{\mathrm n}$ estimated) & std.\ yield & $0.0244$ & $-2.1798$ & $0.0143$ & $0.0244$ & $-3.5985$ & $0.0143$ & $2251$ \\
TPGPMM (Mat\'ern-3/2, regime mean, $\sigma_{\mathrm n}$ estimated) & std.\ yield & $0.0215$ & $-2.2681$ & $0.0129$ & $0.0215$ & $-3.6868$ & $0.0129$ & $2251$ \\
\midrule
Standard GP (RBF) & std.\ yield & $0.0452$ & $-1.6580$ & $0.0250$ & $0.0452$ & $-3.0767$ & $0.0250$ & $2251$ \\
Standard GP (Mat\'ern-3/2) & std.\ yield & $0.0320$ & $-1.9500$ & $0.0181$ & $0.0320$ & $-3.3688$ & $0.0181$ & $2251$ \\
\midrule
laGP & std.\ yield & $0.0919$ & $-1.2189$ & $0.0460$ & $0.0919$ & $-2.6377$ & $0.0460$ & $2251$ \\
HDP-SLDS & $\log(\Delta y_t^2+\epsilon)$ & $1.9097$ & $1.4495$ & $0.7079$ & $0.3094$ & $-1.7894$ & $0.1147$ & $2250$ \\
Mart\'inez--Mena A & std.\ yield & $0.0588$ & $-1.5288$ & $0.0318$ & $0.0588$ & $-2.9476$ & $0.0318$ & $2251$ \\
Mart\'inez--Mena B & std.\ yield & $0.1428$ & $-0.4421$ & $0.0843$ & $0.1429$ & $-1.8608$ & $0.0843$ & $2251$ \\
\bottomrule
\end{tabular}}
\end{table}

\subsection{Out-of-sample forecasting}
\label{subsec:forecast}

Out-of-sample forecasting is the more demanding test. Table~\ref{tab:vix_oos_forecast} evaluates $1$--$5$ trading-day-ahead forecasts of log-VIX
on a held-out $2024$ window: all models are fit \emph{once} on the
$2015$--$2023$ training series and forecast forward with a daily rolling
origin, conditioning on the observed data up to each origin but without
re-estimation. TPGPMM forecasts from each draw's final inferred regime and averages over draws; the standard GPs are single stationary GPs with hyperparameters fitted on the training window.
\begin{itemize}
\item \textbf{The regime-mean alternative} is the most accurate model in CRPS at every horizon and in RMSE from two days ahead. At one day the fitted Mat\'ern GP has a slightly lower RMSE ($\FcStdVixHone$ against $\FcFlexVixHone$). The margins are small: at five days the RMSE/sd is $\FcFlexVixHfive$ against $\FcStdVixHfive$, and the CRPS/sd $\FcFlexCrpsVixHfive$ against $\FcStdCrpsVixHfive$.
\item \textbf{Without a regime mean,} every TPGPMM specification is less accurate than the fitted standard GP with the same kernel, at every horizon. This includes TPGPMM itself, whose RMSE/sd is $\FcMainVixHone$ at one day and $\FcMainVixHfive$ at five, against $\FcStdRVixHone$ and $\FcStdRVixHfive$ for the fitted RBF GP.
\item \textbf{Among the TPGPMM specifications,} all three alternatives with estimated noise forecast more accurately than TPGPMM in CRPS at every horizon. In RMSE they do so at every horizon, except the RBF alternative at four and five days.
\end{itemize}
On VIX, the regime structure therefore matches the best stationary GP in forecasting only when each regime has its own mean and the noise is estimated. The value of TPGPMM itself lies in the segmentation.

\begin{table}[ht]
\centering
\small
\caption{Out-of-sample forecasting on log-VIX. All models are fit once on the $2015$--$2023$ training window and forecast $1$--$5$ trading days ahead across $2024$ ($252$ days) with a daily rolling origin, conditioning on observed data up to each origin but not re-estimated. TPGPMM forecasts are posterior mixtures over 500 draws, each forecasting from its final regime; the standard GPs are single stationary GPs with hyperparameters fitted by maximum marginal likelihood on the training window. Entries are horizon-wise RMSE and CRPS normalised by the test standard deviation; lower is better, best in bold.}
\label{tab:vix_oos_forecast}
\resizebox{\linewidth}{!}{%
\begin{tabular}{lccccc@{\hspace{1.2em}}ccccc}
\toprule
 & \multicolumn{5}{c}{RMSE/sd} & \multicolumn{5}{c}{CRPS/sd} \\
\cmidrule(lr){2-6}\cmidrule(lr){7-11}
Model & $h=1$ & 2 & 3 & 4 & 5 & $h=1$ & 2 & 3 & 4 & 5 \\
\midrule
TPGPMM & 0.706 & 0.893 & 1.069 & 1.237 & 1.400 & 0.369 & 0.473 & 0.574 & 0.671 & 0.761 \\
TPGPMM ($\sigma_{\mathrm n}$ estimated) & 0.636 & 0.846 & 1.046 & 1.243 & 1.435 & 0.322 & 0.436 & 0.544 & 0.649 & 0.752 \\
TPGPMM (Mat\'ern-3/2, $\sigma_{\mathrm n}$ estimated) & 0.548 & 0.780 & 0.978 & 1.146 & 1.288 & 0.272 & 0.397 & 0.501 & 0.590 & 0.674 \\
TPGPMM (Mat\'ern-3/2, regime mean, $\sigma_{\mathrm n}$ estimated) & 0.476 & \textbf{0.664} & \textbf{0.796} & \textbf{0.886} & \textbf{0.940} & \textbf{0.227} & \textbf{0.325} & \textbf{0.392} & \textbf{0.441} & \textbf{0.482} \\
\midrule
Standard GP (RBF) & 0.503 & 0.726 & 0.886 & 0.978 & 1.003 & 0.249 & 0.360 & 0.440 & 0.496 & 0.533 \\
Standard GP (Mat\'ern-3/2) & \textbf{0.470} & 0.670 & 0.808 & 0.898 & 0.946 & 0.230 & 0.332 & 0.402 & 0.455 & 0.496 \\
\bottomrule
\end{tabular}}
\end{table}

On the Treasury yield we follow the same protocol (Table~\ref{tab:treasury_oos_forecast}).
\begin{itemize}
\item \textbf{In RMSE,} the regime-mean alternative and the fitted Mat\'ern GP are close. The GP is more accurate at one and two days ($\FcStdTrHone$ against $\FcFlexTrHone$ at one day), and the regime-mean alternative from three days ahead ($\FcFlexTrHfive$ against $\FcStdTrHfive$ at five days).
\item \textbf{In CRPS,} the fitted Mat\'ern GP is more accurate at every horizon ($\FcStdCrpsTrHfive$ against $\FcFlexCrpsTrHfive$ at five days). The regime-mean alternative's predictive bands are too wide, which the CRPS penalises: all 2024 observations fall inside its $95\%$ five-day interval.
\item \textbf{Within the RBF family,} TPGPMM is less accurate than the fitted RBF GP at one day but more accurate from three days ahead ($\FcMainTrHfive$ against $\FcStdRTrHfive$ at five days); with estimated noise, the RBF version is more accurate at every horizon. The Mat\'ern-3/2 version without a regime mean is less accurate than the fitted Mat\'ern GP at every horizon.
\item \textbf{Among the TPGPMM specifications,} the RBF and regime-mean alternatives forecast more accurately than TPGPMM at every horizon in both metrics. The Mat\'ern-3/2 alternative does so in RMSE at every horizon except four days, but in CRPS only at one and two days.
\end{itemize}
As on VIX, the regime structure gives forecasts comparable to a well-tuned stationary GP rather than clearly better ones. Its advantages are the interpretable segmentation and, for the regime-mean alternative, the fit.

\begin{table}[ht]
\centering
\small
\caption{Out-of-sample forecasting on the standardised 10-year Treasury yield over $2024$ ($250$ days), with the protocol of Table~\ref{tab:vix_oos_forecast}.}
\label{tab:treasury_oos_forecast}
\resizebox{\linewidth}{!}{%
\begin{tabular}{lccccc@{\hspace{1.2em}}ccccc}
\toprule
 & \multicolumn{5}{c}{RMSE/sd} & \multicolumn{5}{c}{CRPS/sd} \\
\cmidrule(lr){2-6}\cmidrule(lr){7-11}
Model & $h=1$ & 2 & 3 & 4 & 5 & $h=1$ & 2 & 3 & 4 & 5 \\
\midrule
TPGPMM & 0.331 & 0.433 & 0.530 & 0.632 & 0.729 & 0.194 & 0.247 & 0.301 & 0.359 & 0.414 \\
TPGPMM ($\sigma_{\mathrm n}$ estimated) & 0.282 & 0.386 & 0.484 & 0.595 & 0.699 & 0.158 & 0.216 & 0.272 & 0.339 & 0.403 \\
TPGPMM (Mat\'ern-3/2, $\sigma_{\mathrm n}$ estimated) & 0.265 & 0.428 & 0.525 & 0.639 & 0.725 & 0.149 & 0.246 & 0.320 & 0.406 & 0.483 \\
TPGPMM (Mat\'ern-3/2, regime mean, $\sigma_{\mathrm n}$ estimated) & 0.250 & 0.383 & \textbf{0.448} & \textbf{0.531} & \textbf{0.595} & 0.141 & 0.222 & 0.274 & 0.330 & 0.377 \\
\midrule
Standard GP (RBF) & 0.291 & 0.447 & 0.597 & 0.761 & 0.909 & 0.161 & 0.246 & 0.331 & 0.429 & 0.517 \\
Standard GP (Mat\'ern-3/2) & \textbf{0.248} & \textbf{0.365} & 0.452 & 0.545 & 0.614 & \textbf{0.137} & \textbf{0.201} & \textbf{0.248} & \textbf{0.304} & \textbf{0.347} \\
\bottomrule
\end{tabular}}
\end{table}

\section{Conclusions}
\label{sec:conclusions}

We introduced the Time-partitioning Gaussian Process Mixture Model, a Bayesian
nonparametric framework that infers a contiguous segmentation of a
time-ordered series jointly with regime-specific Gaussian processes,
learning the kernel and partition hyperparameters, including the Pitman--Yor discount
$\sigma$, directly from the data, with the observation noise fixed at a conservative, data-informed level in the main specification. The model serves two purposes. For regime identification, TPGPMM localises the simulated change points more accurately than every change-point baseline, with an $F_1$ at the level of the best one, Martinez--Mena A, and the alternatives with estimated noise perform comparably. On the log-VIX series it produces a confident segmentation whose boundaries match three of the six benchmark events (the China sell-off, Volmageddon and COVID-19), while shorter-lived spikes are absorbed within regimes; on the Treasury yield it identifies the COVID-19 cuts with certainty and the 2019 easing and the 2022 tightening with uncertain timing. For nonstationary prediction, the specifications that estimate the noise fit better than the change-point baselines, and the regime-mean alternative fits better than a fitted stationary GP and has lower forecast RMSE except at the shortest horizons (Section~\ref{sec:prediction}). The two purposes trade sparsity against predictive accuracy: fixing the noise at a conservative level keeps TPGPMM's segmentation sparse and aligned with the events, while estimating it improves fit and forecasts but either fragments the segmentation or misses events.

For segmentation, the results point to three ingredients. The Pitman--Yor-induced prior on contiguous compositions adapts between the Dirichlet-process limit and heavier-tailed segment-size structures; the non-Markovian GP likelihood absorbs short-range fluctuations as within-regime dynamics and penalises unnecessary boundaries through the marginal likelihood; and the conservative noise level keeps daily fluctuations from being modelled as new regimes. Without any one of them the segmentation changes: estimating the noise with the RBF kernel fragments the series; Martinez--Mena A, with the same prior but Markovian regimes, is as sparse but places its boundaries at short spikes rather than at the China sell-off and COVID-19; Martinez--Mena B, with independent observations, reacts to transient level changes; and BOCPD and HDP-SLDS produce diffuse, frequent boundaries.

TPGPMM is most useful for series that decompose into persistent regimes with non-Markovian within-regime dynamics, where transient shocks should be absorbed rather than flagged and each regime warrants a full functional description rather than only a mean or variance shift, as in regime analysis for risk management and policy evaluation; it is correspondingly less suited to rapidly switching or recurring states. With respect to limitations, we highlight three. First, inference is computationally demanding: each iteration factorises the covariances of the regimes it updates, at a cost cubic in their lengths, so specifications with fewer and longer regimes are slower. On one core, a $60{,}000$-iteration chain on the VIX takes about $\MinPerChainVix$ minutes for TPGPMM, $\MinPerChainVixRbfSn$ and $\MinPerChainVixFlex$ minutes for the RBF and regime-mean alternatives, and $\MinPerChainVixAlt$ minutes for the Mat\'ern-3/2 alternative, which has the fewest and longest regimes. This is practical for series of several thousand observations, which suits macro-financial regime analysis on daily, monthly or quarterly data, whereas multivariate extensions and more complex kernels would raise the cost; we deliberately prioritised an accurate approximation of the posterior (Appendix~\ref{app:exact}) over scaling to very long series. Second, although forecasts average over posterior draws, each draw forecasts from its own final regime, so the predictive mixture reverts towards the recent regime level and cannot anticipate a new regime; the regime-mean alternative improves on a well-tuned stationary GP only by small margins and beyond the shortest horizons, and on the Treasury yield the GP remains more accurate in CRPS, so a persistent component shared across regimes, such as a random-walk or long-range Ornstein--Uhlenbeck kernel, is the natural extension. Third, the composition prior is defined for a fixed series length, so online use would require a sequential construction.

\newpage

\bibliographystyle{plainnat} \bibliography{references}

\newpage

\appendix

\section{Derivation of the Pitman--Yor prior on ordered segmentations}
\label{app:pyp-composition-prior}

A Gibbs-type prior is characterised by an EPPF of the form
\[
  p_K^{(n)}(n_1,\ldots,n_K)
  = V_{n,K}\prod_{k=1}^K (1-\sigma)_{n_k-1},
\]
where $\sigma<1$ is a discount parameter, $(1-\sigma)_{n_k-1}$ denotes the rising factorial $(1-\sigma)(2-\sigma)\cdots(n_k-1-\sigma)$, and the nonnegative weights $V_{n,K}$ satisfy the backward recursion
\[
  V_{n,K} = (n - K\sigma)\,V_{n+1,K} + V_{n+1,K+1},
  \qquad V_{1,1}=1.
\]

Different choices of $(\sigma,V_{n,K})$ recover familiar Bayesian nonparametric priors. In particular:

\noindent
\begin{tabularx}{\textwidth}{X X}
\textbf{Dirichlet process.} For $\sigma=0$,
\[
  V_{n,K} = \frac{\theta^{K-1}}{(\theta+1)_{n-1}},
\]
which implies logarithmic growth of $K$.
&
\textbf{Pitman--Yor process.} For $\sigma\in(0,1)$,
\[
  V_{n,K} = \frac{\prod_{i=1}^{K-1}(\theta+i\sigma)}{(\theta+1)_{n-1}},
\]
which implies polynomial growth of $K$.
\end{tabularx}

For $\sigma \ge 0$, Gibbs-type priors are nonparametric and provide explicit control over the growth of the number of blocks through the discount parameter $\sigma$.

We derive the prior in \eqref{eq:ordered_py_composition_final} from the EPPF of the Pitman--Yor process.
Let $\Pi_n$ denote an exchangeable random partition of $[n]=\{1,\ldots,n\}$
with $K$ blocks, and suppose that the block sizes are
$(n_1,\ldots,n_K)$, where $n_k\ge 1$ and $\sum_{k=1}^K n_k=n$.
For an exchangeable partition with EPPF $p_K^{(n)}$, the probability of any
particular unordered partition of $[n]$ with these block sizes is
\[
  p_K^{(n)}(n_1,\ldots,n_K),
\]
where the EPPF is symmetric in its arguments.

To induce a distribution on ordered compositions, we use the exchangeable
random order construction of \cite{Pitman2006}. Conditional on the unordered
partition and on $K$, the $K$ blocks are placed in a uniformly random order.
Hence each ordered arrangement of the $K$ blocks has conditional probability
$1/K!$.

Now fix an ordered composition
\[
  \rho=(n_1,\ldots,n_K).
\]
The number of ordered assignments of the labels $\{1,\ldots,n\}$ to blocks
of sizes $n_1,\ldots,n_K$ is the multinomial coefficient
\[
  \binom{n}{n_1,\ldots,n_K}
  =
  \frac{n!}{\prod_{k=1}^K n_k!}.
\]
Each ordered assignment corresponds to an unordered exchangeable partition
together with a particular ordering of its $K$ blocks. Therefore, summing over
all ordered label assignments with block-size sequence $\rho$ gives
\[
  \Pr(\rho)
  =
  \binom{n}{n_1,\ldots,n_K}
  \frac{1}{K!}
  p_K^{(n)}(n_1,\ldots,n_K).
\]
This is precisely the induced probability mass function on compositions under
the exchangeable random order construction.

For the Pitman--Yor process with discount parameter $\sigma\in[0,1)$ and
strength parameter $\theta>-\sigma$, the EPPF is
\[
  p_K^{(n)}(n_1,\ldots,n_K)
  =
  \frac{\prod_{i=1}^{K-1}(\theta+i\sigma)}
       {(\theta+1)_{n-1}}
  \prod_{k=1}^K (1-\sigma)_{n_k-1},
\]
with the convention that the empty product equals one when $K=1$.
Substituting this EPPF into the composition probability above yields
\[
  \Pr(\rho)
  =
  \frac{n!}{K!\,\prod_{k=1}^K n_k!}\,
  \frac{\prod_{i=1}^{K-1}(\theta+i\sigma)}
       {(\theta+1)_{n-1}}
  \prod_{k=1}^K (1-\sigma)_{n_k-1}.
\]
This is the expression in \eqref{eq:ordered_py_composition_final}. Summing it over the compositions with $K$ blocks recovers the Pitman--Yor law of $K$; we checked this against the Chinese-restaurant recursion to machine precision.

Finally, in the segmentation model, an ordered composition
$\rho=(n_1,\ldots,n_K)$ determines a unique contiguous segmentation of
$[n]$ by assigning the first block to observations $1,\ldots,n_1$, the second
block to observations $n_1+1,\ldots,n_1+n_2$, and so on. Thus the probability
mass assigned to the ordered composition is inherited directly by the
corresponding contiguous segmentation. The resulting prior preserves the
Pitman--Yor distribution on the number and sizes of blocks while restricting
the support to temporally ordered contiguous regimes.

\paragraph{Two properties.} Writing $w_\sigma(m)=(1-\sigma)_{m-1}/m!$, the prior factorises as
\[
\Pr(\rho\mid\sigma,\theta)=\frac{n!}{(\theta+1)_{n-1}}\cdot\frac{\prod_{i=1}^{K-1}(\theta+i\sigma)}{K!}\cdot\prod_{k=1}^Kw_\sigma(n_k).
\]
The middle factor depends on $K$ only, so the prior is of Gibbs type but not a product partition model. The family is also not projective in $n$: removing the last observation from a draw at size $n+1$ does not give the prior at size $n$, because in a uniformly random order the last block is a uniformly chosen block, whereas the block containing a fixed element of an exchangeable partition is size-biased. For $n=5,8,11$ the total-variation distance between the two laws is $0.06$--$0.10$. The model is therefore defined for a fixed series length, which is why an online version needs a different construction (Section~\ref{sec:conclusions}).

\paragraph{Truncation constant.} The normalising constant of the truncated prior \eqref{eq:truncated_prior} is
\[
Z_n(\sigma,\theta)=\sum_{K\ge1}\frac{n!}{(\theta+1)_{n-1}}\,\frac{\prod_{i=1}^{K-1}(\theta+i\sigma)}{K!}\,A_K(\sigma),
\qquad
A_K(\sigma)=\sum_{\substack{n_1+\cdots+n_K=n\\ n_k\ge m_{\min}}}\ \prod_{k=1}^Kw_\sigma(n_k).
\]
The sums $A_K(\sigma)$, $K=1,\ldots,\lfloor n/m_{\min}\rfloor$, are computed by repeated convolution of the sequence $w_\sigma(m)\mathbf 1\{m\ge m_{\min}\}$ with rescaling, on a grid of $\sigma$ values, and interpolated linearly in $\sigma$; the product is evaluated as $\prod_{i<K}(\theta+i\sigma)=\sigma^{K-1}\Gamma(\theta/\sigma+K)/\Gamma(\theta/\sigma+1)$. The table takes about ten seconds for $n=2{,}264$ and matches brute-force enumeration exactly for $n=12$ and $18$. For $\sigma=0$ and $n\to\infty$, $Z_n\to e^{-1.5\theta}$ when $m_{\min}=3$, so omitting it has a large effect on $\theta$.

\section{Specifications of the within-regime model}
\label{app:specifications}

The four specifications share the partition prior~\eqref{eq:truncated_prior}, the minimum regime length, the priors on $(\sigma,\theta)$ and the prior $\log\ell_k\sim\mathcal N(\log10,0.6^2)$. They differ in the kernel, in whether each regime has its own mean, and in whether the observation noise is fixed or estimated (Table~\ref{tab:specifications}).

\begin{table}[ht]
\centering
\small
\caption{Within-regime specifications. The amplitude prior is $\log s_k^2\sim\mathcal N(0,0.8^2)$ except for the regime-mean alternative, for which it is $\mathcal N(\log\widehat{\mathrm{var}}(\mathbf y),1)$.}
\label{tab:specifications}
\begin{tabular}{llll}
\toprule
Specification & Kernel & Regime mean & Noise level $\sigma_{\mathrm n}$ \\
\midrule
TPGPMM & RBF & no & fixed at $0.1$ \\
TPGPMM ($\sigma_{\mathrm n}$ estimated) & RBF & no & estimated \\
TPGPMM (Mat\'ern-3/2, $\sigma_{\mathrm n}$ estimated) & Mat\'ern-3/2 & no & estimated \\
TPGPMM (Mat\'ern-3/2, regime mean, $\sigma_{\mathrm n}$ estimated) & Mat\'ern-3/2 & yes & estimated \\
\bottomrule
\end{tabular}
\end{table}

\paragraph{Estimated observation noise.} The noise level is shared across regimes, as in~\eqref{eq:gp_observation_model}, with prior
\[
\log\sigma_{\mathrm n}\sim\mathcal N(\log0.1,\,1),
\]
centred on the value fixed in the main specification. Because $\sigma_{\mathrm n}^2$ enters every block covariance, each update of $\sigma_{\mathrm n}$ recomputes all block likelihoods (Appendices~\ref{app:block_C} and~\ref{app:sampler_derivation}).

\paragraph{Mat\'ern-3/2 kernel.} The RBF kernel~\eqref{eq:rbf_kernel} is replaced by
\begin{equation}\label{eq:matern_kernel}
  C_{\boldsymbol{\phi}_k}(x,x')
  = s_k^2\left(1+\frac{\sqrt3\,|x-x'|}{\ell_k}\right)\exp\!\left(-\frac{\sqrt3\,|x-x'|}{\ell_k}\right),
\end{equation}
with the same block hyperparameters $\boldsymbol{\phi}_k=(\ell_k,s_k^2)$ and priors \citep{Stein1999}. Its sample paths are continuous and once mean-square differentiable, whereas those of the RBF kernel are infinitely differentiable, so it can follow rougher within-regime variation. The kernel is stationary, so the block covariances remain symmetric Toeplitz and the sampler of Section~\ref{sec:inference} applies unchanged.

\paragraph{Regime mean.} The regime-mean alternative writes the latent function of block $k$ as a constant level plus a zero-mean GP with the Mat\'ern-3/2 kernel~\eqref{eq:matern_kernel},
\[
  f_k(x)=\mu_k+g_k(x),\qquad
  g_k\mid\boldsymbol{\phi}_k\sim\mathcal{GP}\bigl(0,C_{\boldsymbol{\phi}_k}\bigr),\qquad
  \mu_k\sim\mathcal N(m_0,v_0),
\]
with $\mu_k$ independent across regimes and of $g_k$. Integrating out $\mu_k$ together with $f_k$ replaces the block marginal likelihood~\eqref{eq:block_marginal_likelihood} by
\[
  L_k(\boldsymbol{\phi}_k)=\mathcal N\!\Bigl(\mathbf y_{S_k}\,;\,m_0\mathbf 1,\;K_k(\boldsymbol{\phi}_k)+v_0\mathbf 1\mathbf 1^\top+\sigma_{\mathrm n}^2I_{|S_k|}\Bigr).
\]
Adding the constant $v_0$ to every entry keeps the covariance symmetric Toeplitz, so the sampler is again unchanged. We set $m_0=\bar y$ and $v_0=\widehat{\mathrm{var}}(\mathbf y)$, the sample mean and variance of the series, and centre the amplitude prior on the sample variance, $\log s_k^2\sim\mathcal N(\log\widehat{\mathrm{var}}(\mathbf y),1)$. The predictive distributions of Section~\ref{app:posterior_predictive} apply with $\mathbf y_{S_k}$ replaced by $\mathbf y_{S_k}-m_0\mathbf 1$, every kernel value increased by $v_0$, and $m_0$ added to the predictive mean.

\section{MCMC implementation details}
\label{app:mcmc_details}

This appendix gives the full specification of the four update blocks summarised in Section~\ref{sec:inference}. We write $a_k=\log\ell_k$ and $b_k=\log s_k^2$, and $\varphi_\ell,\varphi_s$ for the normal prior densities of $a_k,b_k$, including their normalising constants. The normalising constants matter because a split or a merge changes the number of hyperparameters. A derivation of every ratio is given in the accompanying note on the sampler equations.

\subsection{Block A: structural updates of $\rho$}
\label{app:block_A}

We enforce a minimum block length $m_{\min}=3$. The structural move at each
iteration is a \emph{split}, which inserts a new internal boundary inside an existing splittable block, or a \emph{merge}, which removes an internal boundary, fusing two adjacent blocks. If $K=1$ a split is proposed; otherwise a split is proposed with probability $q=1/2$ and a merge otherwise. We write $q_K$ for the probability of proposing a split when there are $K$ blocks and $s(\rho)$ for the number of splittable blocks (those with $n_j\ge2m_{\min}$). Each structural move is followed by exact Gibbs draws of two regime boundaries.

\paragraph{Split ($K\to K+1$).}
Choose $j$ uniformly from the splittable blocks and draw the length $c$ of the left child, $m_{\min}\le c\le n_j-m_{\min}$, from
\[
\psi_j(c)=\frac{0.1}{n_j-2m_{\min}+1}+0.9\,\pi_j(c),
\]
where
\[
\pi_j(c)\propto w_\sigma(c)\,w_\sigma(n_j-c)\,\mathcal N(\mathbf y_{1:c}^{(j)};\boldsymbol\phi_j)\,\mathcal N(\mathbf y_{c+1:n_j}^{(j)};\boldsymbol\phi_j)
\]
is the conditional distribution of the split point given the parent's hyperparameters \citep{Zanella2020}; the uniform component keeps every split point reachable. All the prefix and suffix likelihoods come from one Cholesky factorisation of the parent's covariance (Section~\ref{sec:inference}). The children's hyperparameters are
\[
a_L=a_j+w_Ru_\ell,\quad a_R=a_j-w_Lu_\ell,\quad b_L=b_j+w_Ru_s,\quad b_R=b_j-w_Lu_s,\qquad u_\ell,u_s\sim\mathcal N(0,\tau^2),
\]
with $w_L=c/n_j$, $w_R=1-w_L$ and $\tau=0.5$. This map is invertible, with inverse $a_j=w_La_L+w_Ra_R$ and $u_\ell=a_L-a_R$ (and likewise for $b$), and its Jacobian on the log scale is one. The split is accepted with probability $\min\{1,A\}$, where
\begin{multline*}
A=\frac{\theta+K\sigma}{K+1}\,\frac{w_\sigma(c)\,w_\sigma(n_j-c)}{w_\sigma(n_j)}\,
\frac{\varphi_\ell(a_L)\varphi_\ell(a_R)\varphi_s(b_L)\varphi_s(b_R)}{\varphi_\ell(a_j)\varphi_s(b_j)}\,
\frac{\mathcal N(\mathbf y_L;\boldsymbol\phi_L)\,\mathcal N(\mathbf y_R;\boldsymbol\phi_R)}{\mathcal N(\mathbf y_{S_j};\boldsymbol\phi_j)}\\
\times\frac{(1-q_{K+1})/K}{q_K\,\psi_j(c)/s(\rho)}\,\frac{1}{\mathcal N(u_\ell;0,\tau^2)\,\mathcal N(u_s;0,\tau^2)}.
\end{multline*}
The first two factors are the ratio of the partition priors \eqref{eq:truncated_prior}, in which $Z_n$ cancels. Equivalently, in terms of densities on $(\ell,s^2)$ the ratio carries the Jacobian $\ell_L\ell_R/\ell_j\cdot s^2_Ls^2_R/s^2_j$.

\paragraph{Merge ($K\to K-1$).}
Choose an internal boundary uniformly at random and merge the two adjacent
blocks. The merged block receives the size-weighted geometric mean of the
two children's hyperparameters, $a_j=w_La_L+w_Ra_R$ and $b_j=w_Lb_L+w_Rb_R$, and $u=(a_L-a_R,\,b_L-b_R)$. The merge is accepted with probability $\min\{1,A^{-1}\}$, with $A$ evaluated at the split that would recreate the current state from the merged one, including $\psi_j$ and $s(\cdot)$ of the merged composition. Because $\psi_j$ depends only on the pre-split state, which the merge recovers exactly, the ratio is valid for any choice of $\psi_j$.

\paragraph{Boundary Gibbs draws.}
For two internal boundaries chosen at random, the boundary between blocks $i$ and $i+1$, with combined length $T=n_i+n_{i+1}$ starting at position $a$, is drawn from its full conditional
\[
p(c\mid\cdot)\ \propto\ w_\sigma(c)\,w_\sigma(T-c)\,\mathcal N(\mathbf y_{a+1:a+c};\boldsymbol\phi_i)\,\mathcal N(\mathbf y_{a+c+1:a+T};\boldsymbol\phi_{i+1}),\qquad c=m_{\min},\ldots,T-m_{\min},
\]
with the hyperparameters of both blocks held fixed. Unlike a uniform boundary shuffle, this move never proposes an implausible position, and it is cheap: the prefix likelihoods of block $i$ and the suffix likelihoods of block $i+1$ each need a single Cholesky factorisation, so every candidate position is evaluated at the cost of one Metropolis proposal.

\subsection{Block B: kernel hyperparameter update}
\label{app:block_B}

Conditional on the current composition, we select three blocks at random and update $\log\ell_k$ and then $\log s_k^2$ of each by a random-walk Metropolis step,
\begin{equation*}
a_k' = a_k + \lambda_k\varepsilon,\qquad \varepsilon\sim\mathcal{N}(0,1),\qquad \lambda_k=\lambda\sqrt{100/n_k},
\end{equation*}
accepted with probability
\begin{equation}
\alpha = \min\!\Bigl\{1,\ \frac{\varphi_\ell(a_k')\,L_k(e^{a'_k},s_k^2)}{\varphi_\ell(a_k)\,L_k(e^{a_k},s_k^2)}\Bigr\},
\end{equation}
and likewise for $b_k$. Written with densities on $\ell_k$ the ratio carries the Jacobian factor $\ell_k'/\ell_k$. The scale $\lambda$ is adapted towards a 35\% acceptance rate during burn-in and fixed afterwards, so the retained chain is a valid Markov chain. The dominant cost is the $|S_k|\times|S_k|$ Cholesky factorisation required by $L_k$.

\subsection{Block C: Pitman--Yor and noise updates}
\label{app:block_C}

The discount $\sigma$ is updated by a random-walk Metropolis step on the logit scale (proposal SD $0.35$) and the strength $\theta$ on the log scale (proposal SD $0.4$), ten times each per iteration. The acceptance ratios are
\[
\frac{p(\sigma')\,\sigma'(1-\sigma')\,\Pr(\rho\mid\sigma',\theta)\,Z_n(\sigma,\theta)}{p(\sigma)\,\sigma(1-\sigma)\,\Pr(\rho\mid\sigma,\theta)\,Z_n(\sigma',\theta)},\qquad
\frac{p(\theta')\,\theta'\,\Pr(\rho\mid\sigma,\theta')\,Z_n(\sigma,\theta)}{p(\theta)\,\theta\,\Pr(\rho\mid\sigma,\theta)\,Z_n(\sigma,\theta')},
\]
where $\sigma'(1-\sigma')/[\sigma(1-\sigma)]$ and $\theta'/\theta$ are the Jacobians of the logit and log transformations. They involve only the partition prior and the hyperpriors, so no Cholesky factorisations are needed, but they do involve the truncation constant $Z_n$ of Appendix~\ref{app:pyp-composition-prior}.

When the noise is estimated, one random walk on $\log\sigma_{\mathrm n}$ per iteration (step adapted during burn-in) is accepted with the ratio of $\prod_kL_k$ times the normal prior density of $\log\sigma_{\mathrm n}$; it recomputes every block likelihood.

\subsection{Block D: bookkeeping}
\label{app:block_D}

To prevent numerical drift in the per-block log-marginals, the full set of likelihoods is recomputed from scratch every $200$ iterations. We discard the first half of each chain as burn-in and thin the retained chain every five iterations.

\section{Derivation of the sampler}
\label{app:sampler_derivation}

This appendix derives each update of Appendix~\ref{app:mcmc_details} from the detailed-balance condition and identifies, for each, the term whose omission changes the distribution the chain targets. For every such term we state the target that results without it; Table~\ref{tab:sampler_terms} summarises them, and Appendix~\ref{app:diag_verify} reports the tests that confirm the sampler targets~\eqref{eq:posterior}.

\subsection{Target in log coordinates}

We work with $a_k=\log\ell_k$ and $b_k=\log s_k^2$, so the hyperparameter priors are normal densities $\varphi_\ell(a)=\mathcal N(a;\mu_\ell,s_\ell^2)$ and $\varphi_s(b)=\mathcal N(b;\mu_s,s_s^2)$, \emph{including} their normalising constants $(2\pi s_\ell^2)^{-1/2}$ and $(2\pi s_s^2)^{-1/2}$. In these coordinates the posterior~\eqref{eq:posterior}, with $\sigma_{\mathrm n}$ fixed, is
\begin{equation}\label{eq:target_log}
\pi(\rho,a_{1:K},b_{1:K},\sigma,\theta\mid\mathbf y)\ \propto\ p(\sigma)\,p(\theta)\,\frac{\Pr(\rho\mid\sigma,\theta)}{Z_n(\sigma,\theta)}\prod_{k=1}^K\varphi_\ell(a_k)\,\varphi_s(b_k)\,L_k(e^{a_k},e^{b_k}),
\end{equation}
on the union over $K$ of spaces of dimension $2K+2$, with $\rho$ restricted to compositions with blocks of at least $m_{\min}$ observations. Each update below leaves~\eqref{eq:target_log} invariant. The splits and merges connect every composition to the single-block one, so the chain is irreducible on the partitions, and the rejections make it aperiodic.

\subsection{Split and merge}

\paragraph{The reversible-jump condition.} A split maps the state $(\rho,x)$ with $x=(a_{1:K},b_{1:K})$ to $(\rho',x')$ of dimension $2(K+1)$ by drawing auxiliary variables $u\sim g$ and applying a deterministic map $x'=T(x,u)$; the merge applies $T^{-1}$ and draws nothing. Detailed balance between the two moves holds when $T$ is a bijection, the dimensions match ($\dim x+\dim u=\dim x'$), and the split is accepted with probability $\min\{1,A\}$ \citep{green1995},
\begin{equation}\label{eq:green}
A=\frac{\pi(\rho',x')}{\pi(\rho,x)}\;\frac{q_{\mathrm{merge}}(\rho\mid\rho')}{q_{\mathrm{split}}(\rho'\mid\rho)}\;\frac{1}{g(u)}\;\left|\frac{\partial x'}{\partial(x,u)}\right|,
\end{equation}
and the merge with probability $\min\{1,A^{-1}\}$, with $A$ evaluated at the split that would undo it.

\paragraph{Dimension matching.} A split replaces $(a_j,b_j)$ by $(a_L,b_L,a_R,b_R)$, so it needs exactly two auxiliary variables. With $w_L=c/n_j$ and $w_R=1-w_L$,
\[
a_L=a_j+w_Ru_\ell,\qquad a_R=a_j-w_Lu_\ell,\qquad b_L=b_j+w_Ru_s,\qquad b_R=b_j-w_Lu_s .
\]
\begin{proposition}\label{prop:bijection}
The map $(a_j,u_\ell)\mapsto(a_L,a_R)$ is a bijection of $\mathbb R^2$ with inverse $a_j=w_La_L+w_Ra_R$, $u_\ell=a_L-a_R$, and its Jacobian determinant has absolute value one. The same holds for $(b_j,u_s)$.
\end{proposition}
\begin{proof}
$w_La_L+w_Ra_R=a_j+(w_Lw_R-w_Rw_L)u_\ell=a_j$ and $a_L-a_R=(w_R+w_L)u_\ell=u_\ell$, so the stated map is the inverse. The Jacobian matrix is $\begin{pmatrix}1&w_R\\1&-w_L\end{pmatrix}$, whose determinant is $-(w_L+w_R)=-1$.
\end{proof}
The merged block therefore receives the size-weighted geometric means
\[
\ell_j=\ell_L^{w_L}\ell_R^{w_R},\qquad s_j^2=(s_L^2)^{w_L}(s_R^2)^{w_R},
\]
and the merge recovers $u$ exactly. A construction that perturbs each child independently around the parent draws four variables for a move that adds two. Its split and merge are then not inverse maps, \eqref{eq:green} has no valid Jacobian, and the chain does not target the posterior.

\paragraph{Natural scale.} Written in $(\ell,s^2)$, the hyperpriors are log-normal, $p(\ell)=\varphi_\ell(\log\ell)/\ell$, and the Jacobian of the split is $\ell_L\ell_R/\ell_j\cdot s_L^2s_R^2/s_j^2$. The two forms agree:
\[
\frac{p(\ell_L)\,p(\ell_R)}{p(\ell_j)}\cdot\frac{\ell_L\ell_R}{\ell_j}=\frac{\varphi_\ell(a_L)\,\varphi_\ell(a_R)}{\varphi_\ell(a_j)} .
\]

\paragraph{Prior ratio.} From~\eqref{eq:ordered_py_composition_final}, a split increases $K$ by one, multiplies $\prod_{i<K}(\theta+i\sigma)$ by $\theta+K\sigma$, divides by $K+1$ through $K!$, and replaces $w_\sigma(n_j)$ by $w_\sigma(c)\,w_\sigma(n_j-c)$. The truncation constant $Z_n(\sigma,\theta)$ is the same on both sides and cancels. The hyperparameter prior ratio is $\varphi_\ell(a_L)\varphi_\ell(a_R)\varphi_s(b_L)\varphi_s(b_R)/\{\varphi_\ell(a_j)\varphi_s(b_j)\}$, which has one more density factor in $a$ and one in $b$ in the numerator. The normalising constants therefore do \emph{not} cancel.
\begin{proposition}\label{prop:normconst}
If $\varphi_\ell,\varphi_s$ in the split--merge ratio are replaced by unnormalised Gaussian kernels, the chain targets~\eqref{eq:target_log} multiplied by $(2\pi s_\ell s_s)^K$.
\end{proposition}
\begin{proof}
Each split multiplies the ratio by $\sqrt{2\pi}\,s_\ell\cdot\sqrt{2\pi}\,s_s$ and each merge divides it by the same factor, so the modified ratio is the correct ratio for a target that carries this factor once per block.
\end{proof}
With $s_\ell=0.6$ and $s_s=0.8$ the factor is $3.02^K$, a strong bias towards more regimes that does not show up in fixed-dimension moves, where the constants cancel.

\paragraph{Proposal ratio.} The split chooses the move with probability $q_K$, a splittable block with probability $1/s(\rho)$ and the split point with probability $\psi_j(c)$; the merge chooses the move with probability $1-q_{K+1}$ and one of the $K$ boundaries of $\rho'$ uniformly. With $g(u)=\mathcal N(u_\ell;0,\tau^2)\,\mathcal N(u_s;0,\tau^2)$ and unit Jacobian, \eqref{eq:green} gives the ratio $A$ of Appendix~\ref{app:block_A}.

\paragraph{The informed split proposal is valid.} $\psi_j$ depends on the parent block's data, its hyperparameters $\phi_j$ and $\sigma$, all of which belong to the pre-split state. The merge reconstructs that state exactly (Proposition~\ref{prop:bijection}), so the probability of the reverse split can be evaluated at the merge. Detailed balance therefore holds for any $\psi_j$ that gives positive probability to every $c\in\{m_{\min},\ldots,n_j-m_{\min}\}$, which the uniform component guarantees. The choice of $\psi_j$ and of $\tau$ affects only efficiency: $\pi_j(c)$ proposes split points where the data support them, and $\tau=0.5$ lets the children's hyperparameters differ substantially, whereas very small $\tau$ makes most splits indistinguishable from the parent.

\subsection{Boundary moves}

Holding $K$, every other boundary and all hyperparameters fixed, the boundary between blocks $i$ and $i+1$ (combined length $T$, starting after position $a$) can take the values $c\in\{m_{\min},\ldots,T-m_{\min}\}$. In~\eqref{eq:target_log}, $n!$, $(\theta+1)_{n-1}$, $\prod_{i<K}(\theta+i\sigma)/K!$, $Z_n$, the hyperpriors and the likelihoods of all other blocks do not depend on $c$. The full conditional is therefore
\[
\pi(c\mid\text{rest})\ \propto\ w_\sigma(c)\,w_\sigma(T-c)\;\mathcal N(\mathbf y_{a+1:a+c};\boldsymbol\phi_i)\;\mathcal N(\mathbf y_{a+c+1:a+T};\boldsymbol\phi_{i+1}),
\]
a discrete distribution with at most $T$ atoms that can be sampled exactly, so the move is a Gibbs update with no rejection. A uniform shuffle targets the same conditional but proposes most positions where it is negligible, so for long blocks nearly all its proposals are rejected. The Gibbs draw costs about as much as one shuffle proposal because of the following property.
\begin{proposition}\label{prop:prefix}
Let $\Sigma=LL^\top$ be the $m\times m$ covariance of a block with equally spaced inputs and a stationary kernel, $\mathbf v=L^{-1}\mathbf y$ and $\tilde{\mathbf v}=L^{-1}J\mathbf y$, with $J$ the reversal permutation. Then for every $c\le m$,
\[
\log\mathcal N(\mathbf y_{1:c};\mathbf 0,\Sigma_{1:c,1:c})=\sum_{i\le c}\Bigl(-\tfrac12v_i^2-\log L_{ii}-\tfrac12\log2\pi\Bigr),
\]
and the same sum with $\tilde{\mathbf v}$ gives the log-likelihood of the suffix $\mathbf y_{m-c+1:m}$.
\end{proposition}
\begin{proof}
The Cholesky factor of a leading principal sub-matrix is the leading block of $L$, and the leading block of $L^{-1}$ is the inverse of that block, so the first $c$ entries of $\mathbf v$ are the whitened prefix. $\Sigma$ is symmetric Toeplitz, hence persymmetric ($J\Sigma J=\Sigma$), so the reversed suffix has the covariance of a prefix of the same length.
\end{proof}
One factorisation for block $i$ and one for block $i+1$ thus give the likelihood of every candidate position. The same property gives $\pi_j(c)$ for the informed split from a single factorisation of the parent.

\subsection{Kernel hyperparameters}

The random walk proposes $a'_k=a_k+\lambda_k\varepsilon$ with $\varepsilon\sim\mathcal N(0,1)$, which is symmetric in $a$. Metropolis--Hastings with the target density in $a$ gives the ratio $\varphi_\ell(a'_k)L_k(e^{a'_k},s_k^2)/\{\varphi_\ell(a_k)L_k(e^{a_k},s_k^2)\}$. In terms of the log-normal density on $\ell$, the same ratio is $p(\ell'_k)L_k(\ell'_k)\,\ell'_k/\{p(\ell_k)L_k(\ell_k)\,\ell_k\}$.
\begin{proposition}\label{prop:rwjac}
If the random walk on $a=\log\ell$ is accepted with the natural-scale ratio
\[
\frac{p(\ell'_k)\,L_k(\ell'_k)}{p(\ell_k)\,L_k(\ell_k)},
\]
without the factor $\ell'_k/\ell_k$, the chain targets a prior on $\log\ell_k$ shifted from $\mathcal N(\mu_\ell,s_\ell^2)$ to $\mathcal N(\mu_\ell-s_\ell^2,s_\ell^2)$, and likewise for $\log s_k^2$.
\end{proposition}
\begin{proof}
That ratio is the correct one for the density $p(e^a)L(e^a)=\varphi_\ell(a)e^{-a}L(e^a)$ in $a$, and $\varphi_\ell(a)e^{-a}\propto\exp\{-(a-\mu_\ell+s_\ell^2)^2/(2s_\ell^2)\}$.
\end{proof}
With our priors the shifts are $-0.36$ for $\log\ell_k$ and $-0.64$ for $\log s_k^2$: shorter length-scales and smaller amplitudes than the prior states.

The step $\lambda_k=\lambda\sqrt{100/n_k}$ depends on the block length, which this update does not change, so the proposal remains symmetric given $\rho$. The scale $\lambda$ is adapted towards a $35\%$ acceptance rate only during burn-in. After burn-in the transition kernel is fixed, so the retained draws come from a time-homogeneous Markov chain with stationary distribution~\eqref{eq:target_log}, and no conditions on vanishing adaptation are needed.

\subsection{Pitman--Yor parameters and the truncation constant}

Given $\rho$, the conditional of $(\sigma,\theta)$ is proportional to $p(\sigma)\,p(\theta)\Pr(\rho\mid\sigma,\theta)/Z_n(\sigma,\theta)$. The random walks act on $z=\operatorname{logit}\sigma$ and $\log\theta$, whose densities carry the Jacobians $\sigma(1-\sigma)$ and $\theta$. This gives the ratios of Appendix~\ref{app:block_C}. Unlike in the split and merge, $Z_n$ depends on the parameters being updated and does not cancel.
\begin{proposition}\label{prop:zn}
A sampler that omits $Z_n$ from the updates of $(\sigma,\theta)$ targets the posterior under the hyperprior $p(\sigma)\,p(\theta)\,Z_n(\sigma,\theta)$, normalised.
\end{proposition}
\begin{proof}
Omitting $Z_n$ multiplies the conditional density of $(\sigma,\theta)$ by $Z_n(\sigma,\theta)$. This is the same as replacing $p(\sigma)p(\theta)$ by $p(\sigma)p(\theta)Z_n(\sigma,\theta)$ in~\eqref{eq:target_log}, and no other update depends on it.
\end{proof}
$Z_n(\sigma,\theta)$ is the prior probability that every block has at least $m_{\min}$ observations. It falls quickly as $\theta$ grows, because larger $\theta$ produces more and smaller blocks. For $\sigma=0$ the block counts of each size are asymptotically independent Poisson variables with means $\theta/j$ \citep{ArratiaEtAl2003}, so $Z_n\to\exp\{-\theta\sum_{j<m_{\min}}1/j\}=e^{-1.5\theta}$ for $m_{\min}=3$. Under our $\mathrm{Beta}(1,1)\times\mathrm{Gamma}(5,0.5)$ hyperprior, the tilted prior has mean $\sigma\approx0.04$ and mean $\theta\approx2.2$ for both $n=900$ and $n=2{,}264$, instead of $0.5$ and $10$. It pulls both parameters towards the Dirichlet-process limit regardless of the data. In the simulation study, omitting $Z_n$ lowers the average posterior mean of $\sigma$ from $\SigHom$ and $\SigPow$ (Section~\ref{subsec:simulation}) to about $0.03$ and $0.06$, so the adaptation of $\sigma$ to the segment sizes is largely masked. Appendix~\ref{app:pyp-composition-prior} gives the exact computation of $Z_n$.

\subsection{Observation noise}

When $\sigma_{\mathrm n}$ is estimated, the prior is normal on $\log\sigma_{\mathrm n}$ and the random walk is symmetric on that scale. The acceptance ratio is therefore the ratio of $\prod_kL_k$ times the ratio of the normal prior densities, with no Jacobian. Since $\sigma_{\mathrm n}$ enters every block covariance, every block likelihood is recomputed.

\subsection{Summary}

\begin{table}[ht]
\centering
\small
\caption{Terms each update of the sampler must include, the distribution targeted without them, and how the sampler was checked. The last two rows concern efficiency, not correctness.}
\label{tab:sampler_terms}
\begin{tabularx}{\textwidth}{>{\raggedright\arraybackslash}p{3.6cm} >{\raggedright\arraybackslash}X >{\raggedright\arraybackslash}p{3.4cm}}
\toprule
Term & Without it & Check \\
\midrule
Dimension-matched split with exact inverse merge (Prop.~\ref{prop:bijection}) & Split and merge are not inverse maps; the reversible-jump ratio is undefined & Prior-only runs \\
Normalising constants of the hyperpriors (Prop.~\ref{prop:normconst}) & Target multiplied by $(2\pi s_\ell s_s)^K\approx3^K$ & Prior-only law of $K$ \\
Log-scale Jacobian of the random walk (Prop.~\ref{prop:rwjac}) & Prior of $\log\ell_k$ and $\log s_k^2$ shifted by $-s^2$ ($-0.36$, $-0.64$) & Prior-only marginals \\
Truncation constant $Z_n$ in the $(\sigma,\theta)$ updates (Prop.~\ref{prop:zn}) & Hyperprior tilted to mean $\sigma\approx0.04$, $\theta\approx2.2$ & Prior-only KS tests; exact sampler \\
Adaptation during burn-in only & Retained chain is not a fixed Markov kernel & --- \\
\midrule
Gibbs draw of the boundaries (Prop.~\ref{prop:prefix}) & Uniform shuffles, mostly rejected & ESS (Appendix~\ref{app:sampler_diagnostics}) \\
Informed split proposal and $\tau=0.5$ & Uninformed or near-identical splits, mostly rejected & ESS (Appendix~\ref{app:sampler_diagnostics}) \\
\bottomrule
\end{tabularx}
\end{table}

The prior-only runs of Appendix~\ref{app:diag_verify} switch the likelihood off, so the chain must reproduce the prior exactly. A missing term in the first four rows shows up there as a distorted law of $K$ or of the hyperparameters; with all terms included, the total-variation distance from the exact law of $K$ is $\PriorTV$. With the likelihood on, the exact grid sampler of Appendix~\ref{app:exact}, which uses no MCMC moves at all, agrees with the chains on the posterior of $K$ and on the boundary probabilities.

\section{Implementation details of benchmark methods}
\label{app:baselines}

This appendix specifies the implementation used for each benchmark method in Sections~\ref{sec:regimes} and~\ref{sec:prediction}. BOCPD, HDP-SLDS and laGP's in-sample fit are those of the original study of this model, and their results are reproduced from it. Martinez--Mena is refitted for this paper with the settings of \citet{MartinezMena2014}. The standard-GP fits and forecasts were computed for this paper. Each method is applied to the same observation series $\mathbf{y}$ as TPGPMM unless otherwise stated.

\subsection{BOCPD: Bayesian online change-point detection with a
Normal-Inverse-Gamma conjugate underlying predictive model}
\label{app:gpts_cp}

The underlying predictive model of this baseline is a Normal--Inverse-Gamma conjugate model rather than a Gaussian process, so we refer to it as Bayesian online change-point detection \citep[BOCPD;][]{AdamsMacKay2007}. The pipeline follows the ``Industry Portfolios'' specification of \citet{SaatciTurnerRasmussen2010}, with implementation details as in \citet{Turner2011}:

\paragraph{Preprocessing (VIX only).}
Daily close-to-close returns $r_t = \mathrm{VIX}_t / \mathrm{VIX}_{t-1} - 1$ are first standardised using the sample mean $\widehat{\mu}$ and standard deviation $\widehat{\sigma}$ of the first $T_{\mathrm{train}}=250$ observations, $z_t = (r_t - \widehat{\mu})/\widehat{\sigma}$. Each standardised return is then Gaussianised by applying the probability-integral transform under a Student-$t(\nu)$ distribution with $\nu=4$ degrees of freedom, followed by the inverse standard-normal cumulative distribution function (CDF): $\widetilde{z}_t = \Phi^{-1}\!\bigl(F_{t_4}(z_t)\bigr)$. This nonlinear warping maps the heavy-tailed marginal of financial returns into approximate standard normality, ensuring that the Gaussian assumptions of the conjugate model are satisfied. For the simulation study, the synthetic series is already Gaussian by construction and is fed to the model unchanged; no preprocessing is applied.

\paragraph{Within-segment model.}
Within each run, the (Gaussianised) observations are modelled as $\widetilde{z}_t \mid \mu, \sigma^2 \sim \mathcal{N}(\mu, \sigma^2)$, with a Normal-Inverse-Gamma conjugate prior $\mu \mid \sigma^2 \sim \mathcal{N}(\mu_0, \sigma^2/\kappa_0)$ and $\sigma^{-2} \sim \mathrm{Gamma}(\alpha_0, \beta_0)$. We fix $\mu_0 = 0$, $\kappa_0 = 1$, $\alpha_0 = 1$, $\beta_0 = 1$ throughout. Marginalising both $\mu$ and $\sigma^2$, the one-step-ahead predictive distribution is Student-$t$ with degrees of freedom $\mathrm{df}_n = 2\alpha_n$, location $\mu_n$, and scale $\beta_n(\kappa_n + 1)/(\alpha_n \kappa_n)$, where the posterior parameters update analytically from the within-run sufficient statistics.

\paragraph{Run-length recursion.}
We use the standard forward recursion of \citet{AdamsMacKay2007} with a constant hazard $H = 1/\lambda$ and $\lambda = 250$, corresponding to a geometric prior on inter-changepoint intervals with an expected gap of one trading year. The run-length posterior is truncated at $R_{\max} = 1500$, and hypotheses with probability below $10^{-4}$ are pruned at each step.

\paragraph{Change-point detection.}
We declare a change-point at time $t$ when the recent-changepoint score $P(r_t \le 5 \mid x_{1:t})$ exceeds a threshold of $0.5$, with a cooldown of $5$ trading days suppressing duplicate detections within the same event. A short warmup of $k+1=6$ steps is applied so the cumulative score $P(r_t \le k)$ is well-defined.

\subsection{Sticky HDP-SLDS}
\label{app:hdp_slds}

We implement the Markov-switching stochastic volatility (MSSV) variant of the sticky HDP-SLDS of \citet{FoxSudderthJordanWillsky2011}. The latent log-volatility process follows
\[
  x_t = \phi\,x_{t-1} + \mu^{(z_t)} + e_t, \qquad e_t \sim \mathcal{N}(0,\sigma_\eta^2),
\]
with mode-specific process-noise mean $\mu^{(z_t)}$, shared dynamic coefficient $\phi$, and a discrete mode sequence $z_t \in \{1,\ldots,L\}$. The observation model is $y_t = x_t + v_t$, where $v_t$ is approximated by a $10$-component mixture-of-Gaussians fit to the $\log\chi^2(1)$ density. The mode sequence follows a sticky hierarchical Dirichlet process hidden Markov model (HDP-HMM) prior with truncation level $L=15$ and concentration hyperpriors $\alpha + \kappa \sim \mathrm{Gamma}(1, 0.01)$, $\gamma \sim \mathrm{Gamma}(1, 0.01)$, $\rho = \kappa/(\alpha + \kappa) \sim \mathrm{Beta}(10, 1)$. The dynamic-parameter prior is $\phi \sim \mathcal{N}(0, 1)$, $\mu_k \sim \mathcal{N}(0, 1)$, $\sigma_\eta^2 \sim \mathrm{IG}(3, 0.2)$.

\paragraph{Inference.}
A blocked Gibbs sampler alternates between (i) forward-filtering / backward-sampling of $x_{1:T}$ given $z_{1:T}$ and the parameters, (ii) sampling the mixture-component indicators of the measurement noise, (iii) sampling $(\phi, \boldsymbol{\mu}, \sigma_\eta^2)$ given $x_{1:T}$ and $z_{1:T}$, (iv) block sampling $z_{1:T}$ via forward-backward over the truncated HDP transition distributions, and (v) resampling the HDP transition rows and the global stick-breaking weights $\boldsymbol{\beta}$. The simulation study uses $n_{\mathrm{iter}}=1{,}500$ Gibbs iterations with a burn-in of $500$; the VIX application uses $n_{\mathrm{iter}}=3{,}000$ with a burn-in of $1{,}000$. After burn-in, we accumulate the per-time change-point score $\widehat{p}_t = \frac{1}{M}\sum_m \mathbf{1}\{z_t^{(m)} \neq z_{t-1}^{(m)}\}$ and, when needed, the posterior-predictive moments $E[y_t \mid \mathbf{y}]$ and $\mathrm{Var}[y_t \mid \mathbf{y}]$.

\subsection{Martinez--Mena: Pitman--Yor--OU change-point detection}
\label{app:martinez_mena}

We replicate the change-point procedure of \citet{MartinezMena2014}. The
method shares the Pitman--Yor-induced prior on ordered contiguous
compositions used by TPGPMM (Section~\ref{sec:prior_segmentation}), but
combines it with a markedly different within-block model and is fit
through a dedicated split--merge--shuffle MCMC sampler. We evaluate two
specifications from the original paper.

\paragraph{Composition prior.}
The prior on the segmentation $\rho=(n_1,\ldots,n_K)$ is the same
Pitman--Yor-induced ordered-composition prior derived in
Appendix~\ref{app:pyp-composition-prior}, with discount $\sigma$ and strength $\theta$.

\paragraph{Within-block likelihood.}
Conditional on the partition, observations within each block follow a
discretely-observed Ornstein--Uhlenbeck process,
\[
y_t = \mu_k + \phi\,(y_{t-1} - \mu_k) + \xi_t,
\qquad \xi_t \sim \mathcal{N}\!\bigl(0,\,\lambda_k^{-1}(1-\phi^2)\bigr),
\qquad t \in S_k,
\]
with regime-specific mean $\mu_k$ and precision $\lambda_k$, and a shared
autocorrelation parameter $\phi$. The within-block parameters are assigned
the conjugate Normal--Gamma prior
$\mu_k \mid \lambda_k \sim \mathcal{N}\!\bigl(0,\,(c\,\lambda_k)^{-1}\bigr)$,
$\lambda_k \sim \mathrm{Gamma}(a, b)$, and integrated out analytically,
yielding a closed-form per-block marginal likelihood
$L_k(\phi) = p(\mathbf{y}_{S_k} \mid \phi, a, b, c)$ given by Eq.~9 of
\citet{MartinezMena2014}. We use $a = b = 1$ and $c = 0.1$ throughout, as
in the paper.

\paragraph{Inference.}
A split--merge--shuffle MCMC sampler iterates over (i) a structural move on
$\rho$ (split, merge, or boundary shuffle, operating on the integrated
within-block likelihood directly, with no within-block hyperparameters to
perturb), (ii) random-walk Metropolis updates of the autocorrelation
parameter $\phi$ on the logit scale when it is free, and (iii) random-walk
updates of the Pitman--Yor strength $\theta$ on the log scale. The discount
$\sigma$ is held fixed across both specifications, following
\citet{MartinezMena2014}.

\paragraph{Specifications.}
We evaluate the two specifications of the original paper's application \citep[Section~6]{MartinezMena2014}, with its settings. \textbf{Model~A} treats observations as Markovian within each block, with $\phi\sim\mathrm U(0,1)$. \textbf{Model~B} treats them as independent and identically distributed within each block ($\phi=0$ held fixed). In both models the discount is fixed at $\sigma=0.9$ and the strength has the shifted Gamma prior $\theta\sim\mathrm{Ga}(1,1,-0.9)$, that is, $\theta+0.9\sim\mathrm{Gamma}(1,1)$; $a=b=1$ and $c=0.1$; and we run $20{,}000$ iterations after a burn-in of $10{,}000$, for the VIX and Treasury series and for each simulated dataset.

\paragraph{Preprocessing.}
For VIX, $y_t = \log(\mathrm{VIX}_t)$ and the series is centred on its
sample mean before fitting, so the Normal--Gamma prior with prior mean
$\mu_0 = 0$ is well-calibrated. For the simulation study, the synthetic
series $y_t$ is also centred on its sample mean; no further transform is
applied.

\paragraph{Change-point detection.}
We declare a change-point at index $t$ when the marginal posterior
probability that $t$ is an internal boundary,
$\widehat{\tau}_t^{\mathrm{MM}} =
\frac{1}{M^\star}\sum_m \mathbf{1}\{t \text{ is an internal boundary of } \rho^{(m)}\}$,
exceeds the threshold $0.5$, with the same $5$-step cooldown used by
BOCPD and HDP-SLDS.

\paragraph{In-sample predictive.}
For each retained draw, $(\mu_k,\lambda_k)$ of every regime has the Normal--Gamma posterior given the regime's data and the draw's $\phi$, and the predictive of $y_t$ is Student-$t$; the draws are mixed as for TPGPMM. For Model~B ($\phi=0$) the predictive of every $y_t$ in a regime is that of a new observation from the regime. Model~A has no observation noise, so a predictive that conditions on $y_t$ would be degenerate; instead $y_t$ is predicted one step ahead from $y_{t-1}$, with mean $(1-\phi)\,\mathrm{E}[\mu_k]+\phi\,y_{t-1}$, and the first observation of a regime from the regime's stationary distribution. With $\phi=0$ the two coincide.

\subsection{Gaussian-process benchmarks}
\label{app:forecast_baselines}

\paragraph{Standard GP.} A single stationary GP on the whole training series, with RBF or Mat\'ern-3/2 kernel and hyperparameters (length-scale, signal variance, noise) set by maximising the marginal likelihood (ML-II). It is used for the in-sample scores and for the forecasts. The in-sample predictive conditions on the whole training series; for forecasting it conditions on the last 150 observations and is centred on their mean.

\paragraph{laGP.} The local approximate GP of \citet{GramacyApley2015} (R package \texttt{laGP}), with a local design for each prediction point that starts from its $10$ nearest neighbours and is extended to $60$ observations by the ALC-ray criterion, length-scale and nugget by maximum likelihood, and predictions at the training inputs.

\paragraph{TPGPMM forecasts.} For each of 500 posterior draws, the forecast is the GP predictive of that draw's final regime, conditioning on at most the last 150 observations of the regime; specifications without a regime mean are centred on the mean of that window. The draws are averaged, and scores use the mean and variance of the resulting mixture.

\section{Definitions of evaluation metrics}
\label{app:metrics}

This appendix defines the metrics used in Sections~\ref{sec:regimes} and~\ref{sec:prediction}. Throughout, $\mathbf{y} = (y_1,\ldots,y_n)$ is the observed series, and a method's posterior predictive at index $t$ is summarised by its mean $\mu_t$ and variance $\nu_t$.

\subsection{Posterior-predictive metrics}
\label{app:metrics_predictive}

\paragraph{Root mean squared error (RMSE).}
\begin{equation}
  \mathrm{RMSE} = \sqrt{\frac{1}{n}\sum_{t=1}^n (y_t - \mu_t)^2}.
\end{equation}
Same units as $y$. Sensitive only to the predictive mean, not the spread.

\paragraph{Negative log predictive density (NLPD).}
Under a Gaussian predictive $\mathcal{N}(\mu_t, \nu_t)$,
\begin{equation}
  \mathrm{NLPD} = \frac{1}{n}\sum_{t=1}^n \left[ \frac{1}{2}\log(2\pi\nu_t) + \frac{(y_t - \mu_t)^2}{2\nu_t} \right].
\end{equation}
Lower is better. NLPD penalises both location error and miscalibration of the predictive variance.

\paragraph{Continuous ranked probability score (CRPS).}
We use the closed-form Gaussian-approximation CRPS \citep{GneitingRaftery2007},
\begin{equation}
  \mathrm{CRPS}_t = \sigma_t \left[ z_t\bigl(2\Phi(z_t) - 1\bigr) + 2\varphi(z_t) - \frac{1}{\sqrt{\pi}}\right],
\end{equation}
with $\sigma_t = \sqrt{\nu_t}$, $z_t = (y_t - \mu_t)/\sigma_t$, and $\Phi, \varphi$ the standard-normal CDF and probability density function (PDF). The reported CRPS is the average $\frac{1}{n}\sum_t \mathrm{CRPS}_t$. Lower is better. Like NLPD, CRPS is a strictly proper scoring rule that rewards both accurate point predictions and well-calibrated uncertainty, but unlike NLPD it remains finite under heavy-tailed errors. For TPGPMM, $\mu_t$ and $\nu_t$ are the mean and variance of the mixture \eqref{eq:posterior_predictive_mixture}.

\paragraph{Scale-normalised variants.}
Because TPGPMM, laGP, and HDP-SLDS operate on different targets, we report scale-normalised versions to make magnitudes comparable across rows of Tables~\ref{tab:vix_predictive_comparison} and~\ref{tab:treasury_predictive_comparison}. Let $s_y = \mathrm{sd}(y)$ denote the empirical standard deviation of the target. Then
\begin{align}
  \mathrm{RMSE}/\mathrm{sd} &= \frac{\mathrm{RMSE}}{s_y}, \qquad
  \mathrm{CRPS}/\mathrm{sd} = \frac{\mathrm{CRPS}}{s_y}, \\
  \Delta\mathrm{NLPD} &= \mathrm{NLPD} - \left[ \tfrac{1}{2}\log(2\pi s_y^2) + \tfrac{1}{2} \right].
\end{align}
The bracketed term is the NLPD that a marginal-Gaussian baseline $\mathcal{N}(\overline{y}, s_y^2)$ would obtain on its own data. Negative values of $\Delta\mathrm{NLPD}$ therefore indicate that the model improves on the trivial constant-Gaussian fit; RMSE/sd and CRPS/sd less than $1$ indicate improvements over a constant-mean predictor. For the out-of-sample forecasts, $s_y$ is the standard deviation of the 2024 test window.

\subsection{Change-point detection metrics}
\label{app:metrics_cp}

Let $\mathcal{T}^\star = \{t^\star_1,\ldots,t^\star_{K-1}\}$ be the set of true internal change-points and $\widehat{\mathcal{T}} = \{\widehat{t}_1,\ldots,\widehat{t}_{\widehat{K}-1}\}$ the set of inferred change-points. We say that a true change-point $t^\star_j$ is \emph{matched} if there exists $\widehat{t}_i \in \widehat{\mathcal{T}}$ with $|\widehat{t}_i - t^\star_j| \le \tau$, where the matching tolerance $\tau$ is set to $5$ trading days throughout. Each true change-point is matched to at most one inferred change-point and vice versa, in a greedy nearest-neighbour fashion, yielding $n_{\mathrm{matched}}$ matched pairs.

\paragraph{Inferred regime count.}
$\widehat{K} = |\widehat{\mathcal{T}}| + 1$, the pruned number of inferred regimes after applying any method-specific merge radius. For TPGPMM, $\widehat{\mathcal{T}}$ is the set of change-points of the modal composition.

\paragraph{Precision, recall, $F_1$.}
\begin{align}
  \mathrm{Precision} &= \frac{n_{\mathrm{matched}}}{|\widehat{\mathcal{T}}|}, &
  \mathrm{Recall} &= \frac{n_{\mathrm{matched}}}{|\mathcal{T}^\star|}, &
  F_1 &= \frac{2\,\mathrm{Precision}\cdot\mathrm{Recall}}{\mathrm{Precision} + \mathrm{Recall}}.
\end{align}
Precision penalises spurious detections; recall penalises missed change-points. $F_1\in[0,1]$ is their harmonic mean and is the standard one-number summary because precision and recall are in tension---a method can score perfect recall by predicting change-points everywhere, at the cost of precision.

\paragraph{Change-point mean absolute error.}
\begin{equation}
  \mathrm{CP\_MAE} = \frac{1}{|\mathcal{T}^\star|}\sum_{j=1}^{|\mathcal{T}^\star|} \min_{\widehat{t} \in \widehat{\mathcal{T}}} \big|\widehat{t} - t^\star_j\big|.
\end{equation}
The localisation error of the closest inferred change-point to each true change-point, averaged across true change-points. Reported in trading-day units. By convention $\mathrm{CP\_MAE}=0$ when both sets are empty, and undefined ($\mathrm{NaN}$) when $\widehat{\mathcal{T}} = \emptyset$ but $\mathcal{T}^\star \neq \emptyset$.

\subsection{TPGPMM-specific summaries}
\label{app:metrics_pyp}

For the simulation study we additionally summarise the marginal posteriors of the Pitman--Yor parameters from the TPGPMM chains. Given the retained samples $\{(\sigma^{(m)}, \theta^{(m)})\}_{m=1}^{M^\star}$,
\begin{equation}
  \overline{\sigma} = \frac{1}{M^\star}\sum_{m=1}^{M^\star}\sigma^{(m)}, \qquad
  \mathrm{CI}_{95}(\sigma) = \bigl[\sigma^{(0.025)},\,\sigma^{(0.975)}\bigr],
\end{equation}
with $\sigma^{(q)}$ the $q$-th empirical quantile of the chain. The reported $\sigma$ posterior mean in the body text is seed-averaged: for a batch with $R$ replicate seeds, we report $\frac{1}{R}\sum_r \overline{\sigma}_r$. The same convention applies to $\theta$ and $\sigma_{\mathrm n}$ when reported.

\section{Simulation-study configuration}
\label{app:simulation_config}

This appendix gives the exact specification used to generate the synthetic
scenarios in Section~\ref{subsec:simulation}, sufficient to reproduce all
reported TPGPMM numbers from the code release.

\paragraph{Data-generating process.}
For each block $k$, the latent function is sampled from a zero-mean Gaussian
process with squared-exponential kernel of length-scale $\ell_k$ and signal
variance $s_k^2$. The block-specific mean $m_k$ is added, and observations
are drawn as $y_t = f(X_t) + \varepsilon_t$ with
$\varepsilon_t \sim \mathcal{N}(0, \sigma_n^2)$ and $\sigma_n = 0.1$. The
inputs are integer indices $X_t = t-1$. Continuity at boundaries is not
enforced.

The data-generating process includes an explicit per-block mean
$m_k$ in addition to the within-block kernel hyperparameters
$(\ell_k, s_k^2)$, while the fitted TPGPMM has only the kernel
hyperparameters as block-specific parameters. This is deliberate: the
simulation tests whether TPGPMM recovers regime structure when regimes
differ in mean as well as in covariance, a setting strictly more general
than the model's prior. The chain has no parametric mean to fit, so the
mean shift is absorbed by the within-block GP posterior; recovery of the
true partitions in Section~\ref{subsec:simulation} indicates this
absorption is effective for the level scales and segment lengths
considered. The data are generated with a squared-exponential kernel and $\sigma_n=0.1$, so TPGPMM uses the data-generating kernel family and noise level and differs from the data-generating process only in the block means. The Mat\'ern-3/2 alternative with estimated noise is misspecified in the kernel and has to learn the noise, which makes it the more demanding test.

\paragraph{Per-block parameters.}
The two scenarios share identical per-block kernels and means; only the
block lengths differ:

\begin{table}[ht]
\centering
\small
\begin{tabular}{lccc}
\toprule
Block $k$ & $\ell_k$ & $s_k^2$ & $m_k$ \\
\midrule
1  & 25.0 & 0.85 & 0.0  \\
2  &  8.0 & 0.70 & 1.8  \\
3  &  6.0 & 0.75 & $-1.7$ \\
4  & 10.0 & 0.70 & 1.6  \\
5  &  7.0 & 0.70 & $-1.8$ \\
6  & 12.0 & 0.65 & 1.7  \\
7  &  9.0 & 0.75 & $-1.5$ \\
8  & 11.0 & 0.70 & 1.9  \\
9  &  8.0 & 0.70 & $-1.6$ \\
10 &  7.0 & 0.70 & 1.7  \\
\bottomrule
\end{tabular}
\end{table}

\paragraph{Block lengths.}
\begin{itemize}
\item \emph{Homogeneous}: $(n_1,\ldots,n_{10}) = (90,\ldots,90)$, total $n=900$.
\item \emph{Power-law}: $(n_1,\ldots,n_{10}) = (810, 10, 10, 10, 10, 10, 10, 10, 10, 10)$, total $n=900$.
\end{itemize}

\paragraph{Replication and seeds.}
Each scenario is replicated $R=50$ times. For replication $r=1,\ldots,R$,
the data-generating seed is $s_{\mathrm{data}}^{(r)} = 1000 + r$
and the MCMC seed is $s_{\mathrm{model}}^{(r)} = 8800 + r$, for both scenarios. Each replication is fitted with one chain of $15{,}000$ iterations, the first half discarded. Martinez--Mena A and B are fitted on the same datasets, with the same seeds. The BOCPD and HDP-SLDS results in Table~\ref{tab:simulation-study-recovery} are taken from the original study, which used the same design with data seeds $1000$--$1049$; for those two baselines the comparison is therefore between draws from the same data-generating process, not on identical datasets.

\paragraph{Evaluation.}
Change-point matches are declared within $\tau = 5$ trading days. Adjacent
inferred change-points within a merge radius of $4$ days are collapsed before
computing precision, recall, $F_1$, and $\mathrm{CP\_MAE}$.

\section{Sampler diagnostics}
\label{app:sampler_diagnostics}

To assess sampler mixing over the partition space, we report acceptance
rates for the split and merge moves, the fraction of boundary Gibbs draws that move a boundary, effective sample sizes for the global
quantities $K$, $\sigma$, and $\theta$, and the rank-normalised split-$\widehat{R}$ across four chains started from very different partitions. Detailed numbers are collected in
Table~\ref{tab:sampler_diagnostics}; the traces of the global
quantities for the VIX chains are shown in Figure~\ref{fig:trace_diag}.

\begin{table}[ht]
\centering
\small
\caption{Sampler diagnostics. Four chains per series started at $K_0\in\{1,5,15,30\}$; acceptance rates are averaged over chains; ESS is the rank-normalised bulk effective sample size pooled over chains and $\widehat{R}$ the rank-normalised split-$\widehat{R}$ \citep{VehtariEtAl2021}, reported as the maximum over $K,\sigma,\theta$; statistics on the second half of each chain.}
\label{tab:sampler_diagnostics}
\resizebox{\linewidth}{!}{%
\begin{tabular}{llcccccccc}
\toprule
Series & Specification & acc$_{\mathrm{split}}$ & acc$_{\mathrm{merge}}$ & moved$_{\mathrm{Gibbs}}$ & ESS$_K$ & ESS$_\sigma$ & ESS$_\theta$ & $\widehat{R}$ (max) & min/chain \\
\midrule
VIX & TPGPMM & 0.055 & 0.055 & 0.439 & 143 & 7336 & 3013 & 1.031 & 8.3 \\
VIX & TPGPMM ($\sigma_{\mathrm n}$ estimated) & 0.065 & 0.064 & 0.381 & 12 & 14083 & 51 & 1.266 & 5.3 \\
VIX & TPGPMM (Mat\'ern-3/2, $\sigma_{\mathrm n}$ estimated) & 0.061 & 0.061 & 0.354 & 544 & 3452 & 10183 & 1.006 & 50.0 \\
VIX & TPGPMM (Mat\'ern-3/2, regime mean, $\sigma_{\mathrm n}$ estimated) & 0.147 & 0.145 & 0.452 & 88 & 8470 & 280 & 1.047 & 6.5 \\
\midrule
Treasury & TPGPMM & 0.126 & 0.127 & 0.780 & 387 & 1762 & 4441 & 1.011 & 13.8 \\
Treasury & TPGPMM ($\sigma_{\mathrm n}$ estimated) & 0.095 & 0.091 & 0.539 & 26 & 4604 & 158 & 1.139 & 2.0 \\
Treasury & TPGPMM (Mat\'ern-3/2, $\sigma_{\mathrm n}$ estimated) & 0.115 & 0.115 & 0.667 & 147 & 1655 & 4148 & 1.023 & 18.1 \\
Treasury & TPGPMM (Mat\'ern-3/2, regime mean, $\sigma_{\mathrm n}$ estimated) & 0.062 & 0.062 & 0.597 & 96 & 2952 & 1417 & 1.031 & 6.7 \\
\bottomrule
\end{tabular}}
\end{table}

\begin{figure}[t]
  \centering
  \includegraphics[width=0.9\linewidth]{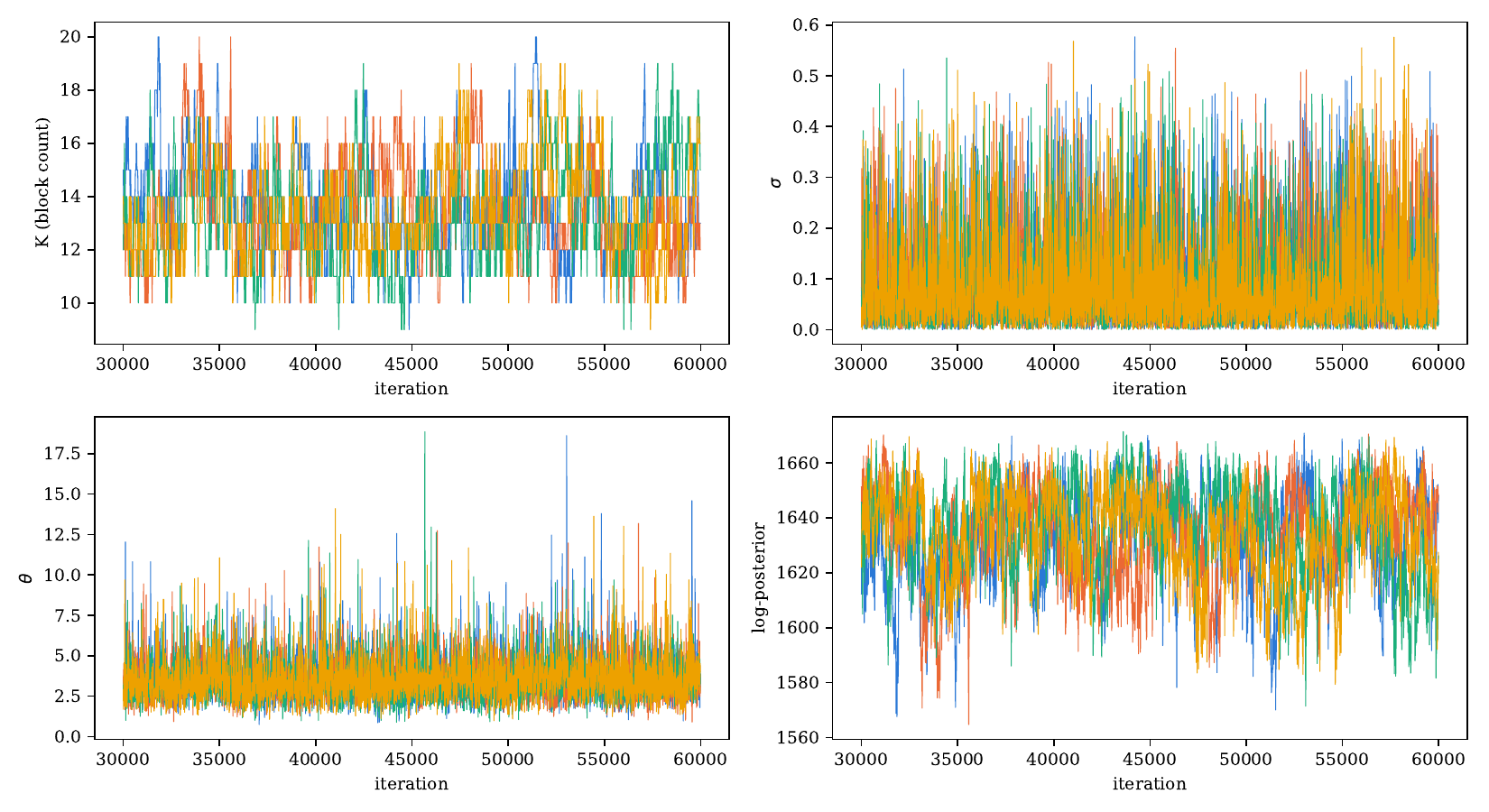}
  \caption{Post-burn-in traces of the global quantities ($K$, $\sigma$,
  $\theta$, and the log-posterior) for the four VIX chains of TPGPMM, started from $K_0=1,5,15,30$ equal blocks (first half of the $60{,}000$ iterations discarded as burn-in). The four chains overlap in every panel.}
  \label{fig:trace_diag}
\end{figure}

For TPGPMM the four chains agree: the largest $\widehat{R}$ over $K,\sigma,\theta$ is $\RhatMainVix$ on VIX and $\RhatMainTr$ on the Treasury yield, and the effective sample size of $K$ is $\EssKvix$ and $\EssKtr$. The same holds for the Mat\'ern-3/2 and regime-mean alternatives. The RBF model with estimated noise is the exception ($\widehat{R}=\RhatRbfSn$ on VIX and $\RhatRbfSnTr$ on the Treasury yield, with $\mathrm{ESS}_K=\EssKRbfSn$ on VIX): its chains settle at different combinations of the noise level and the number of regimes. This is the weak separability of within-regime amplitude and noise that motivates fixing $\sigma_{\mathrm n}$ (Section~\ref{sec:gp_likelihood}).

For comparison, a basic split--merge sampler for TPGPMM, with uniform split points, a uniform boundary shuffle and hyperparameter perturbations of scale $0.05$ at a split, gives $\mathrm{ESS}_K=5$ and split-$\widehat{R}_K=1.51$ for a $20{,}000$-iteration VIX chain. Three features account for the difference: the exact Gibbs draws of the boundaries, the informed split proposal, and the dimension-matched split with auxiliary variables of scale $\tau=0.5$.

\subsection{Verification of the sampler}
\label{app:diag_verify}

\paragraph{Prior-only runs.} With the likelihood switched off ($n=20$, noise estimated), four chains of $60{,}000$ iterations reproduce the prior: the total-variation distance between the sampled and the exact law of $K$ is $\PriorTV$, and the Kolmogorov--Smirnov distances of $\sigma$ and $\theta$ from their priors are $0.019$ and $0.002$. A sampler that omits $Z_n$ gives mean $\theta\approx3.0$ and mean $\sigma\approx0.15$ in the same setting, against the prior means $10$ and $0.5$. Block likelihoods from the Toeplitz prefix computation match a direct multivariate normal density to relative error $10^{-13}$ for all kernels, with and without the regime mean.

\paragraph{Exact sampler.} With the likelihood switched on, Appendix~\ref{app:exact} compares the chains with an exact sampler that needs no MCMC moves. The two agree on the posterior of the number of regimes, the boundary probabilities and the Pitman--Yor parameters.

\section{Validation against an exact sampler}
\label{app:exact}

For TPGPMM's main specification, with an RBF kernel and fixed noise, the posterior of the partition can be computed without MCMC once the kernel hyperparameters are restricted to a grid. Its draws are independent, so it has no convergence issue, and its only approximation is the grid. It therefore provides a reference against which to check that the chains of Appendix~\ref{app:mcmc_details} converge to the right distribution.

\paragraph{Discretised hyperprior.} We replace the log-normal priors on $(\ell_k,s_k^2)$ by a discrete prior on a regular grid $\{(a_g,b_g)\}_{g=1}^G$ of $(\log\ell,\log s^2)$, covering $\mu_\ell\pm3.3s_\ell$ and $[\mu_s-3s_s,\mu_s+4s_s]$, with weights $\pi_g\propto\varphi_\ell(a_g)\varphi_s(b_g)$. Because the block hyperparameters are independent across blocks, they can be summed out block by block. Each block $y_{s+1:t}$ then contributes the collapsed likelihood
\[
S(s,t)=\sum_{g=1}^G\pi_g\,\mathcal N\bigl(\mathbf y_{s+1:t};\mathbf 0,\Sigma_g[1{:}t-s,\,1{:}t-s]\bigr),
\]
and the posterior of $(\rho,\sigma,\theta)$ is
\[
\pi(\rho,\sigma,\theta\mid\mathbf y)\ \propto\ p(\sigma)\,p(\theta)\,\frac{c_K(\sigma,\theta)}{Z_n(\sigma,\theta)}\prod_{k=1}^Kw_\sigma(n_k)\,S(i_{k-1},i_k),\qquad
c_K(\sigma,\theta)=\frac{n!}{(\theta+1)_{n-1}}\,\frac{\prod_{i=1}^{K-1}(\theta+i\sigma)}{K!}.
\]
As the grid is refined, this converges to the posterior of the continuous model~\eqref{eq:target_log}.

\paragraph{All block likelihoods at once.} Every block covariance at grid point $g$ is a leading sub-matrix of the same $n\times n$ Toeplitz matrix $\Sigma_g$. Let $L_g$ be its Cholesky factor and $H$ the Hankel matrix whose column $s$ holds $(y_{s+1},\ldots,y_n,0,\ldots,0)$. By Proposition~\ref{prop:prefix}, the cumulative sums down each column of the matrix with entries
\[
-\tfrac12\bigl[(L_g^{-1}H)_{rs}\bigr]^2-\log(L_g)_{rr}-\tfrac12\log2\pi
\]
are the log-likelihoods of all blocks $(s,t)$. One factorisation and one triangular solve per grid point therefore give the whole matrix $S$.

\paragraph{Forward recursion and backward sampling.} The prior is a product of block terms times a function of $K$ (Appendix~\ref{app:pyp-composition-prior}), so the recursion carries the number of blocks:
\[
F_1(t)=w_\sigma(t)\,S(0,t),\qquad F_K(t)=\sum_{s:\,t-s\ge m_{\min}}F_{K-1}(s)\,w_\sigma(t-s)\,S(s,t),
\]
and
\[
p(\mathbf y\mid\sigma,\theta)=\frac{1}{Z_n(\sigma,\theta)}\sum_Kc_K(\sigma,\theta)\,F_K(n).
\]
We place $(\sigma,\theta)$ on a grid of $24$ values of $\sigma\in[0,0.9]$ and $120$ log-spaced values of $\theta\in[0.05,60]$, weighted by their priors. Each independent draw takes three steps \citep{Fearnhead2006}:
\begin{enumerate}
\item draw $(\sigma,\theta)$ from its posterior on the grid;
\item draw $K$ with probability proportional to $c_K(\sigma,\theta)F_K(n)$;
\item draw the boundaries backwards. The last block starts at $s$ with probability
\[
\Pr(s)\ \propto\ F_{K-1}(s)\,w_\sigma(n-s)\,S(s,n),
\]
and the earlier boundaries follow in the same way.
\end{enumerate}
We use $4{,}000$ draws and three grids, with $G=\ExCoarseGVix$, $783$ and $\ExFineGVix$ points. On $12$ cores the block likelihoods take $7$--$51$ seconds and a complete run two to three minutes.

\paragraph{Results.} Table~\ref{tab:exact_check} and Figure~\ref{fig:exact_check} compare the exact posterior with TPGPMM's four chains.
\begin{itemize}
\item \textbf{Number of regimes.} On VIX, the exact posterior mean of $K$ is $13.48$ on the $783$-point grid and $\ExFineKVix$ on the finest grid, against $\ExMcmcKVix$ for the chains, whose Monte Carlo standard error is $\ExMcseVix$. The total-variation distance between the posteriors of $K$ from the finest grid and from the chains is $\ExTvVix$. On the Treasury yield the two means are $\ExFineKTr$ and $\ExMcmcKTr$ (standard error $\ExMcseTr$), with a total-variation distance of $\ExTvTr$.
\item \textbf{Boundary locations.} The probabilities of a boundary within $\pm10$ trading days of each date agree closely along both series: the mean absolute difference is $\ExDmeanVix$ on VIX and $\ExDmeanTr$ on the Treasury yield, and the largest is $\ExDmaxVix$ and $\ExDmaxTr$.
\item \textbf{Pitman--Yor parameters.} The posterior means of $\sigma$ and $\theta$ are $\ExSigVix$ and $\ExThVix$ (exact) against $\ExSigMcmcVix$ and $\ExThMcmcVix$ (chains) on VIX, and $\ExSigTr$ and $\ExThTr$ against $\ExSigMcmcTr$ and $\ExThMcmcTr$ on the Treasury yield.
\item \textbf{Effect of the grid.} The coarsest grid, with $\ExCoarseGVix$ points, overstates the number of regimes on VIX ($\ExCoarseKVix$). The two finer grids differ from each other by about as much as they differ from the chains, so the remaining gap is of the order of the discretisation error. On the Treasury yield all three grids agree with the chains.
\end{itemize}
The chains therefore converge to the exact posterior, up to Monte Carlo error and the discretisation of the hyperprior. The check is limited to the main specification. With estimated noise, $\sigma_{\mathrm n}$ would add a third grid dimension, and the cost of the block-likelihood matrix grows linearly with the grid size.

\begin{table}[ht]
\centering
\small
\caption{TPGPMM's chains against the exact grid sampler, which includes the truncation constant. $\mathrm{TV}_K$ is the total-variation distance between the two posteriors of $K$, and $\max|\Delta\tau|$ the largest difference between the probabilities of a boundary within $\pm10$ trading days of each date. Exact results use $4{,}000$ independent draws.}
\label{tab:exact_check}
\resizebox{\linewidth}{!}{%
\begin{tabular}{llccccc}
\toprule
Series & Method & $E[K]$ (95\% interval) & $E[\sigma]$ & $E[\theta]$ & $\mathrm{TV}_K$ & $\max|\Delta\tau|$ \\
\midrule
VIX & Exact, 210-point grid & 14.43 [11, 18] & 0.072 & 3.67 & 0.227 & 0.24 \\
VIX & Exact, 783-point grid & 13.48 [10, 18] & 0.083 & 3.51 & 0.029 & 0.15 \\
VIX & Exact, 1{,}760-point grid & 13.69 [11, 18] & 0.082 & 3.58 & 0.049 & 0.14 \\
VIX & TPGPMM chains (MCMC) & 13.44 [11, 17] & 0.090 & 3.56 & --- & --- \\
\midrule
Treasury & Exact, 210-point grid & 7.04 [5, 10] & 0.359 & 5.01 & 0.045 & 0.09 \\
Treasury & Exact, 783-point grid & 6.97 [5, 10] & 0.367 & 5.17 & 0.030 & 0.03 \\
Treasury & Exact, 1{,}760-point grid & 6.95 [5, 10] & 0.367 & 5.16 & 0.024 & 0.04 \\
Treasury & TPGPMM chains (MCMC) & 7.00 [5, 10] & 0.375 & 5.26 & --- & --- \\
\bottomrule
\end{tabular}}
\end{table}

\begin{figure}[ht]
  \centering
  \includegraphics[width=\linewidth]{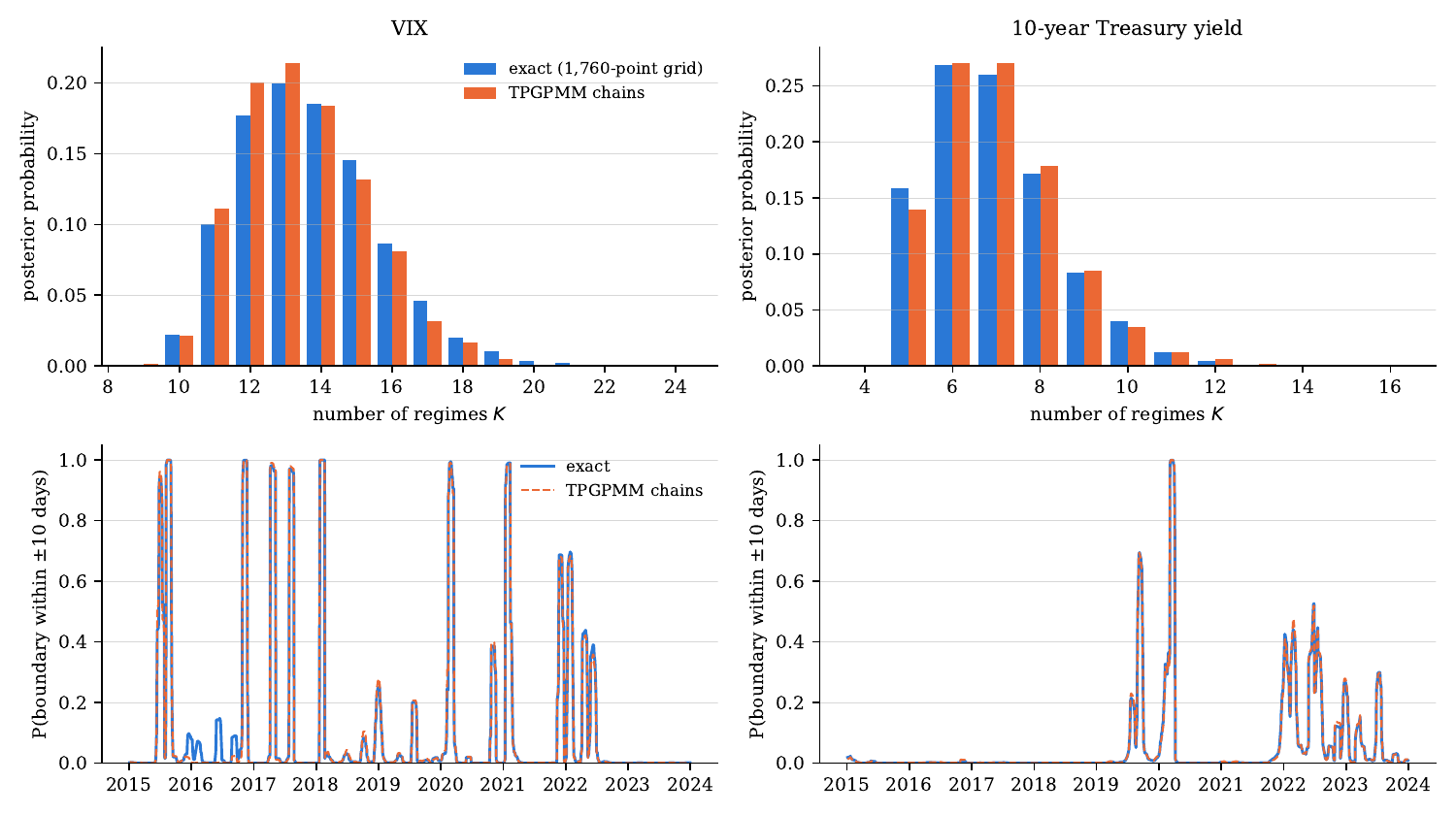}
  \caption{Exact grid sampler against TPGPMM's chains, for VIX (left) and the 10-year Treasury yield (right). \emph{Top}: posterior of the number of regimes $K$. \emph{Bottom}: posterior probability of a boundary within $\pm10$ trading days of each date. The exact results use the finest grid and $4{,}000$ independent draws.}
  \label{fig:exact_check}
\end{figure}

\end{document}